\documentclass[10pt]{article}

\usepackage[margin=1in]{geometry}
\usepackage[utf8]{inputenc}
\usepackage{amsmath,amssymb,amsfonts,amsthm,mathtools}
\newcommand{\toprule}{\hline}
\newcommand{\midrule}{\hline}
\newcommand{\bottomrule}{\hline}
\usepackage{graphicx}
\usepackage{hyperref}
\usepackage{xcolor}
\usepackage{array}
\usepackage{enumitem}
\usepackage{verbatim}
\usepackage{placeins} % keep floats within their section (prevents figures drifting past appendix boundaries)
\usepackage{listings}

\definecolor{codebg}{RGB}{248,248,248}
\definecolor{codeframe}{RGB}{200,200,200}
\definecolor{codekeyword}{RGB}{0,0,180}
\definecolor{codestring}{RGB}{163,21,21}
\definecolor{codecomment}{RGB}{0,128,0}
\definecolor{codenumber}{RGB}{128,128,128}

\lstdefinestyle{pythonstyle}{
  backgroundcolor=\color{codebg},
  frame=single,
  rulecolor=\color{codeframe},
  basicstyle=\ttfamily\scriptsize,
  keywordstyle=\color{codekeyword}\bfseries,
  stringstyle=\color{codestring},
  commentstyle=\color{codecomment}\itshape,
  numbersep=8pt,
  showstringspaces=false,
  breaklines=true,
  breakatwhitespace=false,
  tabsize=4,
  columns=fullflexible,
  keepspaces=true,
}
\usepackage{authblk}
\usepackage{algorithm}
\usepackage{algpseudocode}
\usepackage{cleveref}
\usepackage[backend=biber, sorting=none]{biblatex}
\hypersetup{colorlinks=true,linkcolor=blue,citecolor=blue,urlcolor=blue}

\newtheorem{theorem}{Theorem}[section]
\newtheorem{lemma}{Lemma}[section]
\newtheorem{proposition}{Proposition}[section]
\newtheorem{corollary}{Corollary}[section]

\newcommand{\peaklr}{\eta_{\mathrm{peak}}}
\newcommand{\etamax}{\eta_{\max}}

\title{Scale Weight Decay and Train Better}
\author{Anuj Apte\thanks{\texttt{anuj.apte@jpmchase.com}, \texttt{apte.anuj@outlook.com}}}
\affil{Global Technology Applied Research, JPMorganChase,
New York, NY 10001}
\date{\vspace{-8.5ex}}

\begin{document}
\maketitle

\begin{abstract}
\normalsize
The discovery of scaling laws has motivated training neural networks on ever increasing quantities of data. This is typically done with a constant decoupled weight decay which causes the network weights to shrink steadily over the course of training. Taking inspiration from the Robbins--Monro conditions, we propose to scale weight decay by the fraction of the peak learning rate $\eta/\eta_{\max}$. We prove that this scaled weight decay preserves the asymptotic stationarity guarantees of the corresponding unregularized methods for both stochastic gradient descent and the non-Euclidean spectral optimizer Muon, thereby avoiding the additional asymptotic bias introduced by constant decoupled weight decay. This retains the stability benefits of weight decay without changing the asymptotic optimization target. Using a steady-state analysis, we explain why under standard weight decay the weight norm shrinks steadily as training proceeds, whereas under scaled weight decay it settles to a roughly constant value. When applied to the training of mixture-of-experts models, Muon with scaled weight decay (Muon-SW) consistently outpaces Muon with identical hyperparameters, reaching the same validation loss $\mathbf{30\%}$ faster at our largest scale across models from $72 - 930$ million parameters trained at $\sim 600$ tokens per active parameter. If this trend continues to hold, the method promises to substantially accelerate the pre-training of frontier models while requiring only a few lines of code to implement.

\end{abstract}

\begin{figure}[htb]
    \centering
    \includegraphics[width=0.9\linewidth]{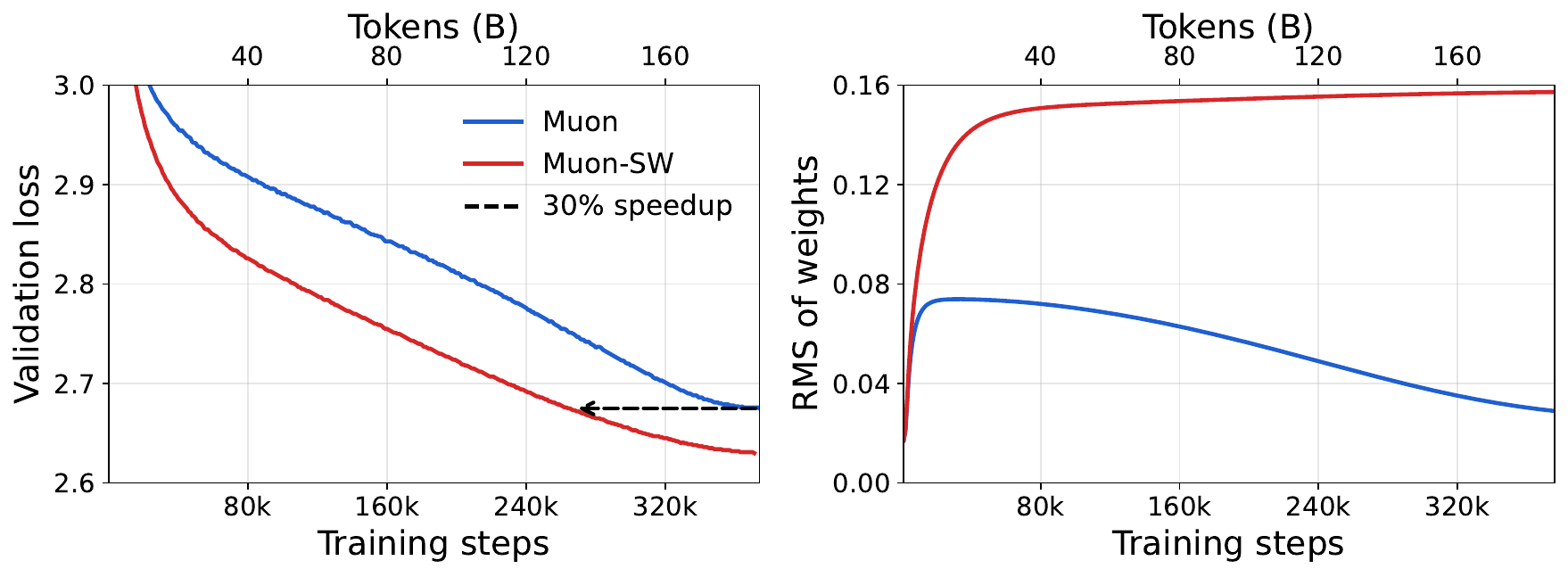}
    \caption{\textbf{Scaled weight decay trains faster and keeps weights from collapsing (width-1024 MoE, ${\sim}932$M parameters).} \emph{Left:} validation loss versus optimizer steps for \textcolor{blue}{Muon} with standard weight decay and \textcolor{red}{Muon-SW} with scaled weight decay (decay factor $\textcolor{red}{\eta_t^2 \lambda/\eta_{\max}}$ instead of $\textcolor{blue}{\eta_t \lambda}$). Muon-SW reaches the best validation loss of the standard run ($2.675$) in $264$k steps versus $374$k, a $\sim 30\%$ reduction in steps (dashed arrow). \emph{Right:} RMS norm of the weights, averaged across non-router layers. Under standard weight decay the norm peaks early and then steadily decays, whereas under scaled weight decay it settles to a roughly constant value. The models are trained on FineWeb dataset \cite{penedo2024finewebdatasetsdecantingweb} and use 
   identical hyperparameters: $8$ experts, top-2 routing, weight decay $0.1$, a peak learning rate of $4.8\times10^{-3}$ that decays to $4.8\times10^{-4}$ by the end of training, and a global batch of $256$ sequences of length $2{,}048$, corresponding to $524{,}288$ tokens per step.}
    \label{fig:w1024-main}
\end{figure}

\newpage

\section{Introduction}

Scaling laws, which predict neural network loss as a function of training data and model size, have motivated the training of neural networks on increasingly larger datasets \cite{kaplan2020scalinglawsneurallanguage,hoffmann2022trainingcomputeoptimallargelanguage,bahri2024}. This trend holds across modalities beyond language including vision \cite{dosovitskiy2021imageworth16x16words, zhai2022}, audio \cite{manakul2026} and robotics \cite{sartor2025}. For example, the largest language models with publicly available training methodology are pre-trained on tens of trillions of tokens \cite{kimiteam2026kimik2openagentic, 5team2025glm45agenticreasoningcoding, deepseekai2026deepseekv4highlyefficientmilliontoken}. This pre-training phase is where the model learns to predict the next token as accurately as possible. 
A key role in the training process is played by the weight decay which is applied throughout. Weight decay was originally introduced to aid in regularization and prevent overfitting by modifying the loss $\mathcal{L}$ as $\mathcal{L} + \lambda ||W||^2 /2 $ where $\lambda$ is the coefficient of weight decay and $\ell_2$ norm is used to measure the size of the weights $W$ \cite{krogh1991}. However, for modern deep learning the number of training data points far exceeds the number of weights which makes overfitting unlikely. In this case, the primary benefits of weight decay are stabilizing the training process by modifying the dynamics of training and reducing the final loss \cite{dangelo2024needweightdecaymodern}. Weight decay is applied in a decoupled manner given by 

\begin{equation}\label{eq:decoupled-wd}
 W_{t+1} = (1 - \textcolor{blue}{\eta_t \lambda}) W_{t}- \eta_{t} U_t~,
\end{equation}

where $W_t, \eta_t, U_t$ are the weight, learning rate and pre-conditioned gradient update at step $t$ and $\lambda$ is the coefficient of weight decay \cite{loshchilov2019decoupledweightdecayregularization}. For optimizers which pre-condition the update such as AdamW, this is not equivalent to modifying the loss function with an $\ell_2$ regularization term. Thus, contraction is applied separately from the adaptive gradient normalization. Since the coefficient of weight decay is multiplied by the learning rate $\eta_t$, the cumulative effect of weight decay depends on the choice of the learning schedule. For training transformers and their variants, it is important to increase the learning rate in the beginning till it reaches its maximum value \cite{kalra2024warmuplearningrateunderlying}. This process is called warmup; after the end of warmup the learning rate is decreased as training proceeds, typically following a cosine schedule \cite{loshchilov2017sgdrstochasticgradientdescent}.

An important change that has occurred since 2025 is the increased adoption of Muon as the main optimizer over AdamW. Muon is an optimizer for matrix- or tensor-valued weights that computes the update subject to a spectral norm constraint. This is accomplished by orthogonalizing the momentum buffer and then scaling by the learning rate before updating the weight \cite{jordan2024muon, liu2025muonscalablellmtraining}. Following AdamW, Muon incorporates a decoupled weight decay with a constant coefficient $\lambda$. As a result, the weights of the network continue to shrink in magnitude after the end of warmup as illustrated by the blue curve in the right panel of Figure~\ref{fig:w1024-main}, where we have plotted the root mean squared (RMS) value of matrix-valued weights averaged over the layers as a function of training steps. If the primary role of weight decay in modern deep learning is not regularization but controlling the training dynamics and improving the loss, then it is no longer prudent to apply weight decay with a constant coefficient which continuously pulls the training towards the origin. By plugging in $W_{t+1}=W_{t}$ in \eqref{eq:decoupled-wd} we observe that if the training converges it does to a point where $W_t = -U_t/ \lambda$, which might be quite far from an optimum of the loss function. Consequently, with constant decoupled weight decay the convergence guarantees apply to modified optimization dynamics rather than to a minimizer of the original loss function. 

This motivates the following question:

\begin{center}
\textit{Can weight decay be applied in a principled manner that empirically improves the training process while retaining the stationarity guarantees of the corresponding unregularized methods?}
\end{center}

Taking inspiration from the Robbins--Monro conditions \cite{Robbins1951}, we propose that instead of a constant weight decay it should be scaled by the fraction of the peak learning rate $\eta_t/\eta_{\max}$, where $\eta_{\max}$ is the maximum learning rate; that is, the base decay $\eta_t\lambda$ is further multiplied by $\eta_t/\eta_{\max}$

\begin{equation}\label{eq:decoupled-swd}
 W_{t+1} = \bigg(1 - \textcolor{red}{\frac{\eta_t^2}{\eta_{\max}}\lambda} \bigg) W_{t}- \eta_{t} U_t~.
\end{equation}

As a result, the impact of weight decay decreases as learning rate decreases towards the end of training. Using this method, we answer the above question in the affirmative and make the following contributions: 

\begin{itemize}[leftmargin=*, align=left]
\item By incorporating scaled weight decay into Muon, we derive Muon-SW and show that it substantially speeds up training of mixture-of-experts (MoE) \cite{shazeer2017, fedus2022} transformer models of size ranging from $72 - 930$ M parameters trained on $18-200$ B tokens corresponding to $600$ tokens per active parameter. The speed-up generally increases with scale and is largest at width $1024$, where Muon-SW reaches the same validation loss $\mathbf{30\%}$ faster than Muon with identical hyperparameters, as reported in Table~\ref{tab:main-results}.

\item Employing a steady-state analysis, we explain why under standard weight decay the weight norm shrinks steadily after warmup, whereas under scaled weight decay it settles to a roughly constant value. This is illustrated in the right panel of Figure~\ref{fig:w1024-main} and in Figure~\ref{fig:w1024-align}. 

\item We prove that, for both SGD (stochastic gradient descent) and Muon, scaled weight decay retains the asymptotic stationarity guarantees of the corresponding unregularized methods under the same conditions (Theorems~\ref{thm:scaled-wd-sgd-convergence} and~\ref{thm:scaled-wd-muon-stationarity}). In contrast, we show that constant weight decay can remain bounded away from a minimizer on simple quadratic objectives (Propositions~\ref{prop:const-wd-nonconvergence} and~\ref{prop:const-wd-muon}), as illustrated in Figure~\ref{fig:convergence-illustration}.

\end{itemize}

Defazio~\cite{defazio2025} introduces the same rule to counteract the growth of gradient norms late in training, a motivation that does not apply to update-normalizing optimizers like Muon. Instead, we motivate the rule through Robbins--Monro summability and show that the scaled shrink term preserves the asymptotic stationarity guarantee for the original unregularized objective. We also conduct a thorough empirical investigation across model sizes and training horizons for MoE models trained with Muon. Since the speedups are substantial, if our observed trend continues then scaled weight decay offers substantial gains for frontier training while requiring only a few lines of code to implement and incurring no additional computational overhead. Although we develop and evaluate scaled weight decay for Muon, the same principle should apply to other optimizers for matrix-valued weights such as Shampoo \cite{gupta2018, anil2021}, SOAP \cite{vyas2025}, Aurora \cite{dewulf2026auroraleverageawarespectraloptimizer}, and other variants of Muon \cite{ahn2025dion, ahn2025dion2,amsel2026polarexpressoptimalmatrix,apte2026anytimetrainingschedulefreespectral,an2025asgoadaptivestructuredgradient,liu2026muon2boostingmuonadaptive,khaled2025muonbpfastermuonblockperiodic}.

\section{Scaled Weight Decay and Muon-SW}
\label{sec:scaled-wd}

Consider an unconstrained optimization problem of minimizing an objective function $f: \mathcal{W} \to \mathbb{R}$ over a space $\mathcal{W}$ such as $\mathbb{R}^d$ or $\mathbb{R}^{m \times n}$:
\begin{equation}\label{eq:sw-problem}
    \min_{W \in \mathcal{W}} f(W)~.
\end{equation}
The primary setting of interest is when $\mathcal{W}$ denotes the parameter (weight) space of a neural network, and the objective $f$ is the expected cross-entropy loss under the data-generating distribution
\begin{equation}
f(W) = \mathbb{E}_{z \sim \mathcal{D}}\big[\ell(W; z)\big]~,
\end{equation}
where $z$ is a data point drawn from data distribution $\mathcal{D}$ and $\ell(W; z)$ is the per-sample cross-entropy loss \cite{bottou2018optimization}. By backpropagation through the network, we can compute a noisy estimate of the gradient of $f$ over a mini-batch 
\begin{equation}\label{eq:swd-noisy-grad}
    G_t = \nabla f(W_t) + \xi_t~,
\end{equation}
where $\xi_t$ is assumed to be a zero-mean noise term with bounded second moment. This gradient estimate $G_t$ is used to compute the pre-conditioned update $U_t$, which in general is a function of the current and past gradients. Typically the gradients are first accumulated into a momentum buffer $M_t$ to average out the noise, and $U_t$ is then obtained by applying an optimizer-specific preconditioner to $M_t$. For AdamW this is a coordinate-wise rescaling by a running estimate of the second moment, while for Muon it is an orthogonalization that normalizes the update to have unit spectral norm.

When the optimizer is SGD, the preconditioner is the identity, and the iteration $W_{t+1} = W_t - \eta_t G_t$ is exactly the stochastic approximation scheme studied by Robbins and Monro \cite{Robbins1951}. Assuming $f$ is $L$-smooth, bounded below, and the noise $\xi_t$ has bounded variance, the method admits an asymptotic stationarity guarantee provided the learning-rate schedule $\{\eta_t\}$ satisfies the two conditions
\begin{equation}\label{eq:robbins-monro}
    \sum_{t} \eta_t = \infty
    \qquad \text{and} \qquad
    \sum_{t} \eta_t^2 < \infty~.
\end{equation}
The first condition prevents the total step budget from becoming finite, while the second ensures that the accumulated variance of the noise $\xi_t$ is finite. Under the Robbins--Monro conditions, the unregularized method admits an asymptotic stationarity guarantee. Provided the iterates remain bounded, an additional update term of order $\mathcal{O}(\eta_t^2)$ is summable and preserves this guarantee. In the strongly convex setting, the guarantee can be strengthened to convergence to the unique minimizer.

A term that scales as $\mathcal{O}(\eta_t)$, on the other hand, can move the stationary set. This is precisely the situation for standard decoupled weight decay \eqref{eq:decoupled-wd}, which enters at order $\mathcal{O}(\eta_t)$ and can converge to a point that is not a minimizer of the original objective $f$. If instead we scale the decay by the learning-rate fraction $\eta_t / \etamax$ as given in \eqref{eq:decoupled-swd}, the shrink term enters at order $\mathcal{O}(\eta_t^2)$. Thus, scaling weight decay preserves the stabilizing effect on the training dynamics while the learning rate is high and retains the asymptotic stationarity guarantee for the original unregularized loss (Section~\ref{sec:convergenc-theory}).

\begin{algorithm}[htb]
\caption{Muon-SW}
\label{alg:muon-swd}
\setlength{\baselineskip}{1.2\baselineskip}
\begin{algorithmic}[1]
  \State \textbf{Input:} initial weights $W_0 \in \mathbb{R}^{m \times n}$, loss $\mathcal{L}$, base learning-rate schedule $\{\eta_t\}$, momentum $\beta$, weight decay $\lambda$, peak learning rate $\etamax$
\State Initialize $M_0 \in \mathbb{R}^{m \times n} \gets 0$
\For{$t = 1, 2, \dots$}
    \State $G_t \gets \nabla_W \mathcal{L}(W_t)$
    \State $M_t \gets \beta M_{t-1} + G_t$
    \State $O_t \gets \mathrm{NS5}(M_t)$
    \State $\widehat{\eta}_t \gets 0.2\, \eta_t \sqrt{\max(m,n)}$
    \State $W_{t+1} \gets \bigl(1 - \textcolor{red}{\lambda \eta_t^2/\etamax}\bigr) W_t - \widehat{\eta}_t\, O_t$
\EndFor
\end{algorithmic}
\end{algorithm}

% In Algorithm~\ref{alg:muon-swd}, $\eta_t$ is the base learning-rate schedule, while $\widehat{\eta}_t = 0.2\,\eta_t\sqrt{\max(m,n)}$ is the actual matrix-scaled Muon update magnitude. The decay factor is computed from the base learning rate $\eta_t$.

We now specialize to Muon, which is an optimizer for matrix-valued parameters. Note that most of the weights of a modern neural network are matrix- or tensor-valued. For example in a transformer, the overwhelming majority of the trainable parameters live in the two-dimensional weight matrices of the attention and feed-forward blocks, with only a small fraction residing in parameters such as embeddings, normalization scales, and biases. Muon updates each matrix-valued weight by taking the momentum buffer and replacing it with the update of bounded spectral norm that is best aligned with it. Let $M_t$ be the momentum buffer given by
\begin{equation}\label{eq:muon-momentum}
    M_t = \beta M_{t-1} + G_t~,
\end{equation}
and consider the update direction that maximizes alignment with $M_t$ subject to a spectral-norm constraint,
\begin{equation}\label{eq:muon-constraint}
    O_t = \arg\max_{\|O\|_2 \le 1} \langle O, M_t \rangle~.
\end{equation}
Writing the singular value decomposition $M_t = U \Sigma V^\top$, the maximizer of \eqref{eq:muon-constraint} is obtained by setting all nonzero singular values to $1$, giving the orthogonalized direction
\begin{equation}\label{eq:muon-orthogonalize}
    O_t = U V^\top = \mathrm{polar}(M_t)~,
\end{equation}
which is the polar factor of $M_t$. In practice the exact SVD is avoided and $O_t$ is approximated by a fixed number of Newton--Schulz iterations applied to $M_t$. The pre-conditioned update is then $U_t = O_t$, with the matrix-shape rescaled learning rate represented by $\widehat{\eta}_t$. The Muon update with decoupled (constant) weight decay is
\begin{equation}\label{eq:muon-update}
    W_{t+1} = (1 - \textcolor{blue}{\eta_t \lambda}) W_t - \widehat{\eta}_t\, O_t~.
\end{equation}
Substituting the scaled decay coefficient of \eqref{eq:decoupled-swd} in place of the constant $\lambda$ in \eqref{eq:muon-update} gives the Muon variant with scaled weight decay, which we call Muon-SW:
\begin{equation}\label{eq:muon-swd-update}
    W_{t+1} = \Bigl(1 - \textcolor{red}{\frac{\eta_t^2}{\etamax} \lambda}\Bigr) W_t - \widehat{\eta}_t\, O_t~.
\end{equation}

We establish the convergence of this update rule in Section~\ref{sec:convergenc-theory}, and study its effect on training dynamics in Section~\ref{sec:steady-state}. The convergence analysis absorbs the fixed matrix-shape scaling into the effective step-size notation, so $\eta_t$ there denotes the effective step size.

Since Muon is applied only to matrix-valued weights, Muon-SW must be used alongside another optimizer, such as AdamW, for the remaining parameters. Rescaling the learning rate by a factor $0.2 \sqrt{\max{(m,n)}}$ for a matrix of shape $m \times n$ allows one to match the RMS of the updates from the two optimizers and hence reuse the same base learning rate across both \cite{liu2025muonscalablellmtraining}. The pseudocode for Muon-SW is presented in Algorithm \ref{alg:muon-swd}; the full implementation is deferred to Appendix~\ref{app:implementation}. The major change relative to Muon is the factor $\lambda \eta_t^2/\etamax$ multiplying the decay coefficient, and orthogonalization uses five Newton--Schulz iterations (NS5).

\section{Convergence Theory}
\label{sec:convergenc-theory}

We have motivated scaling the weight decay with $\eta / \etamax$ based on the Robbins--Monro conditions. However, strictly speaking they apply to cases where the training continues for an unbounded period of time which is necessary to ensure that the optimizer can converge starting from an arbitrary initial point. For a finite training horizon one can still obtain convergence rates in terms of the total training time $T$, though these are less clean than the asymptotic Robbins--Monro statements. The results below exhibit the asymptotic contrast: constant weight decay can remain separated from a minimizer of the original loss, whereas scaled weight decay retains stationarity guarantees for the original objective and converges to its minimizer in the strongly convex SGD setting.

\subsection{Constant weight decay converges to a non-minimizer}
Recall from the introduction that with constant decay the iteration balances the gradient update against the shrink term, so at convergence $W_t = -U_t/\lambda$. Because the original loss $f$ need not be small at this shifted point, constant weight decay can settle at a higher loss than the true minimum. 

To make this precise we exhibit a simple one-dimensional example on which SGD with constant weight decay provably fails to converge to the minimizer. We take the objective to be deterministic so that no stochasticity is involved, which only strengthens the conclusion. Consider the shifted quadratic
\begin{equation}\label{eq:shifted-quadratic}
    f(w) = \frac{a}{2}(w - c)^2, \qquad a > 0,\ c \neq 0~,
\end{equation}
whose unique minimizer is $w^\star = c$. Gradient descent with constant decoupled weight decay $\lambda > 0$ gives the update
\begin{equation}\label{eq:const-wd-affine}
    w_{t+1} = \bigl(1 - \eta_t(a + \lambda)\bigr) w_t + \eta_t\, a c~.
\end{equation}

\begin{proposition}\label{prop:const-wd-nonconvergence}
Let the learning-rate schedule satisfy $0<\eta_t\le 1/(a+\lambda)$ and $\sum_t \eta_t = \infty$. Then from every starting point $w_0$, the iterates \eqref{eq:const-wd-affine} converge to
\begin{equation}\label{eq:const-wd-fixed-point}
    \bar{w} = \frac{ac}{a + \lambda} \neq c~,
\end{equation}
with loss gap
\begin{equation}\label{eq:const-wd-loss-gap}
    f(\bar{w}) - f(c) = \frac{a}{2}\left(\frac{\lambda}{a + \lambda}\right)^2 c^2 > 0~.
\end{equation}
In particular the iterates never converge to the minimizer $w^\star = c$. 
\end{proposition}

The proof is given in Appendix~\ref{app:const-wd-1d}. The hypothesis $\sum_t \eta_t = \infty$ is exactly the first Robbins--Monro condition, and it is the natural regime in which the optimizer is allowed to travel arbitrarily far and thus has any chance of converging from an arbitrary initialization. In this regime, the failure in Proposition \ref{prop:const-wd-nonconvergence} is a property of the dynamics rather than of the initialization or the horizon. If instead $\sum_t \eta_t < \infty$ the iterate simply runs out of step budget and freezes at a point determined by the schedule and $w_0$; it can then land on the minimizer only for a finely tuned starting point, which is a measure-zero set. 

The same failure occurs for Muon in a qualitatively different form. Consider the shifted matrix quadratic
\begin{equation}\label{eq:matrix-quadratic}
    f(W) = \frac{1}{2}\|W - A\|_F^2, \qquad \nabla f(W) = W - A~,
\end{equation}
with $A \in \mathbb{R}^{m\times n}$ and unique minimizer $W^\star = A$. Muon with momentum $\beta \in [0,1)$ forms
\begin{equation}\label{eq:muon-momentum-theory}
    M_t = \beta M_{t-1} + \nabla f(W_t), \qquad O_t = \mathrm{polar}(M_t)~,
\end{equation}
where $\mathrm{polar}(C) = UV^\top$ for a compact singular value decomposition $C = U\Sigma V^\top$, so that $\|O_t\|_{\mathrm{op}} \le 1$. The constant-decay Muon update is
\begin{equation}\label{eq:muon-const-wd-theory}
    W_{t+1} = (1 - \eta_t\lambda) W_t - \eta_t O_t~.
\end{equation}

\begin{proposition}\label{prop:const-wd-muon}
Let $\lambda > 0$, suppose $0 < \eta_t\lambda \le 1$ and $\sum_t \eta_t = \infty$, and consider the Muon iteration \eqref{eq:muon-momentum-theory}--\eqref{eq:muon-const-wd-theory} on the objective \eqref{eq:matrix-quadratic}. Then, regardless of the momentum coefficient $\beta$ and the initialization $W_0$,
\begin{equation}\label{eq:muon-radius}
    \limsup_{t\to\infty} \|W_t\|_{\mathrm{op}} \le \frac{1}{\lambda}~.
\end{equation}
Consequently, if $\sigma_{\max}(A) > 1/\lambda$, then
\begin{equation}\label{eq:muon-separation}
    \liminf_{t\to\infty} \|W_t - A\|_{\mathrm{op}} \ge \sigma_{\max}(A) - \frac{1}{\lambda} > 0~,
\end{equation}
so the iterates never converge to the minimizer $W^\star = A$, and the loss stays bounded below by
\begin{equation}\label{eq:muon-loss-gap}
    \liminf_{t\to\infty} f(W_t) \ge \frac{1}{2}\sum_{i}\Bigl(\sigma_i(A) - \frac{1}{\lambda}\Bigr)_+^2~,
\end{equation}
where $\sigma_i(A)$ are the singular values of $A$ and $(x)_+ = \max(x,0)$.
\end{proposition}

The proof is given in Appendix~\ref{app:const-wd-muon}. In contrast to the scalar case, the orthogonalized Muon direction has spectral norm at most $1$ irrespective of the gradient magnitude. A constant decay therefore confines every singular value of $W_t$ to a ball of radius $1/\lambda$ set by $\lambda$ alone, and any target outside that ball is missed. For training of transformers, spectral norms larger than $1/\lambda$ can arise in the MLP blocks (see Figure 1 in \cite{lion2026}).

Proposition~\ref{prop:const-wd-muon} also shows that the convergence guarantee for Muon with constant weight decay claimed by Sato et al.\ \cite{sato2026muon} is incorrect. Their Lemma~C.1 in fact bounds $\|W_t\|_{\mathrm{op}}$ by $1/\lambda$ rather than $1$, so the strictly positive descent coefficient their Theorem~C.1 relies on is unavailable when $\lambda < 1$; taking $A$ with $\sigma_{\max}(A) > 1/\lambda$ gives an explicit counterexample. 

\begin{figure}[htb]
    \centering
    \includegraphics[width=0.9\linewidth]{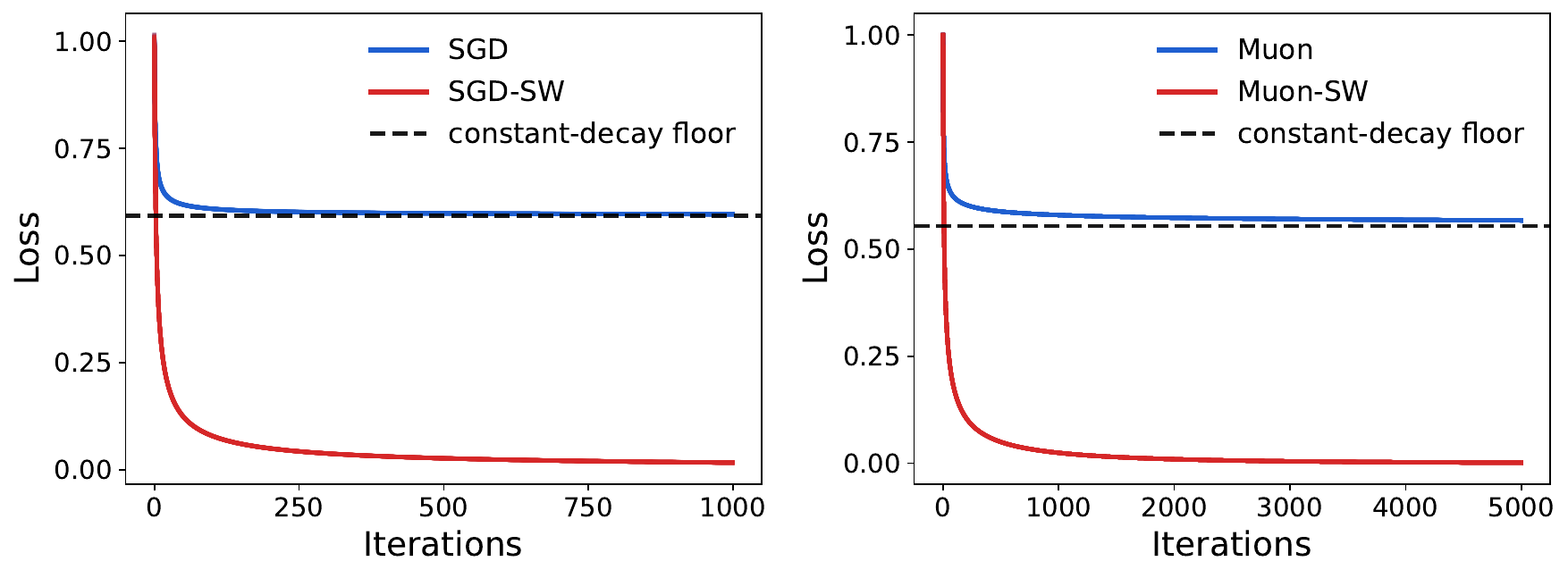}
    \caption{Scaled weight decay converges to a minimizer of the original loss, while constant weight decay does not. \textbf{Left:} deterministic gradient descent on the shifted scalar quadratic $f(w) = \frac{a}{2}(w-c)^2$ of Proposition~\ref{prop:const-wd-nonconvergence}. \textbf{Right:} Muon on the shifted matrix quadratic $f(W) = \frac{1}{2}\|W - A\|_F^2$ of Proposition~\ref{prop:const-wd-muon}, where the target $A$ has a Marchenko--Pastur-type spectrum with several singular values above $1/\lambda$. In each panel we plot the loss (normalized by its mean initial value) for many random initializations $W_0 \sim \mathcal{N}(0, 0.02)$ under the same $\eta_t = \eta_0/(1+t)$ schedule. The scaled-decay runs (SGD-SW, Muon-SW; red) drive the loss to zero, whereas the constant-decay runs (SGD, Muon; blue) instead flatten at the analytic lower bound (dashed).}
    \label{fig:convergence-illustration}
\end{figure}

\subsection{Convergence of scaled weight decay for SGD}
\label{subsec:scaled-wd-convergence}
In contrast to the behavior for constant weight decay, we now show that scaled weight decay preserves the standard weighted stationarity guarantee for the original objective. We begin with the stochastic gradient descent (SGD) case, where the argument is the simplest. Let $F : \mathbb{R}^d \to \mathbb{R}$ be bounded below by $F_{\inf}$ and $L$-smooth. At each step we observe a stochastic gradient $g_t$ satisfying
\begin{equation}\label{eq:sgd-assumptions}
    \mathbb{E}[g_t \mid \mathcal{F}_t] = \nabla F(x_t), \qquad
    \mathbb{E}[\|g_t\|_2^2 \mid \mathcal{F}_t] \le \sigma^2 + \rho \|\nabla F(x_t)\|_2^2~,
\end{equation}
for constants $\sigma^2 \ge 0$ and $\rho \ge 1$, where $\mathcal{F}_t$ is the filtration generated by the iterates and gradients up to step $t$. Scaled weight decay with coefficient $\lambda > 0$ uses the update
\begin{equation}\label{eq:scaled-wd-sgd}
    x_{t+1} = x_t - \eta_t g_t - q\, \eta_t^2 x_t, \qquad q = \frac{\lambda}{\etamax}~.
\end{equation}

\begin{theorem}\label{thm:scaled-wd-sgd-convergence}
Suppose $F$ and $g_t$ satisfy \eqref{eq:sgd-assumptions}, the iterates are bounded in the sense $\sup_t \mathbb{E}\|x_t\|_2^2 < \infty$, and the learning-rate schedule satisfies the Robbins--Monro conditions $\sum_t \eta_t = \infty$ and $\sum_t \eta_t^2 < \infty$. Then the scaled weight decay iterates \eqref{eq:scaled-wd-sgd} satisfy
\begin{equation}\label{eq:scaled-wd-sgd-limit}
    \frac{1}{S_T}\sum_{t=0}^{T-1} \eta_t\, \mathbb{E}\|\nabla F(x_t)\|_2^2 \to 0,
    \qquad S_T = \sum_{t=0}^{T-1}\eta_t~,
\end{equation}
and in particular $\liminf_{t\to\infty} \mathbb{E}\|\nabla F(x_t)\|_2^2 = 0$. Thus scaled weight decay has the same asymptotic weighted stationarity guarantee for the original objective $F$ as SGD without weight decay under the corresponding assumptions.
\end{theorem}

When the objective is in addition strongly convex, the stationarity guarantee upgrades to convergence to the unique minimizer, with an explicit rate, and the bounded-iterate assumption is no longer needed since strong convexity controls the distance to $x^\star$ directly.

\begin{theorem}\label{thm:scaled-wd-sgd-strongly-convex}
Suppose, in addition to \eqref{eq:sgd-assumptions}, that $F$ is $\mu$-strongly convex with $\mu \le L$ and unique minimizer $x^\star$, and that the Robbins--Monro conditions hold. Then the scaled weight decay iterates \eqref{eq:scaled-wd-sgd} satisfy
\begin{equation}\label{eq:scaled-wd-sgd-strongly-convex-limit}
    \mathbb{E}\|x_t - x^\star\|_2^2 \to 0~,
\end{equation}
and for the harmonic schedule $\eta_t = \gamma/(t+t_0)$ with $\mu\gamma > 1$ the rate is $\mathbb{E}\|x_t - x^\star\|_2^2 = O(1/t)$.
\end{theorem}

The proof of both theorems is given in Appendix~\ref{app:scaled-wd-sgd}. The key point is that the scaled shrink term $q\eta_t^2 x_t$ enters the descent inequality only at order $\eta_t^2$, so it contributes only a summable perturbation to the stationarity analysis. Applying the strongly convex result to the shifted quadratic recovers convergence to the true minimizer, in direct contrast to the constant-decay failure of Proposition~\ref{prop:const-wd-nonconvergence}.

\begin{corollary}\label{cor:scaled-wd-quadratic}
Let $f(w) = a(w-c)^2/2$ with $a > 0$, and consider deterministic gradient descent with scaled weight decay, $w_{t+1} = w_t - \eta_t\, a(w_t - c) - q\eta_t^2 w_t$, under a Robbins--Monro schedule. Then $w_t \to c = w^\star$, the minimizer of the original objective.
\end{corollary}

\begin{proof}
The objective is $L$-smooth and $\mu$-strongly convex with $L = \mu = a$, so Theorem~\ref{thm:scaled-wd-sgd-strongly-convex} applies with deterministic gradients ($\sigma^2 = 0$, $\rho = 1$) and gives $w_t \to c$.
\end{proof}

\subsection{Convergence of scaled weight decay for Muon}
\label{subsec:scaled-wd-muon}

We now establish the analogous positive result for Muon. The setting is a matrix objective $f : \mathbb{R}^{m\times n} \to \mathbb{R}$ that is bounded below by $f_{\inf}$ and $L$-smooth in the Frobenius norm, with a stochastic gradient $G_t$ satisfying $\mathbb{E}[G_t \mid \mathcal{F}_t] = \nabla f(W_t)$ and $\mathbb{E}[\|G_t - \nabla f(W_t)\|_F^2 \mid \mathcal{F}_t] \le \sigma^2/B$, where $B$ is a fixed batch size. Muon with scaled weight decay forms the momentum buffer $M_t = \beta M_{t-1} + G_t$ with fixed $\beta \in [0,1)$, orthogonalizes it to $O_t = \mathrm{polar}(M_t)$, and updates
\begin{equation}\label{eq:muon-scaled-wd-theory}
    W_{t+1} = W_t - \eta_t O_t - q\, \eta_t^2 W_t, \qquad q = \frac{\lambda}{\etamax}~.
\end{equation}
We write $r = \min(m,n)$ for the ambient rank bound. Unlike SGD the natural stationarity measure for Muon is the \emph{nuclear norm} $\|\nabla f(W_t)\|_*$ rather than the Frobenius norm. When the orthogonalized direction is formed from the true gradient, spectral--nuclear duality gives $\langle \nabla f(W_t), O_t\rangle_F = \|\nabla f(W_t)\|_*$. With momentum and stochastic gradients this equality is perturbed by the buffer-tracking error, which is controlled in the appendix. For fixed momentum $\beta$ and a fixed batch $B$, the noise cannot be fully averaged out, so the guarantee holds only up to an explicit noise floor, which is also the case for Muon without any weight decay.

\begin{theorem}\label{thm:scaled-wd-muon-stationarity}
Suppose the iterates remain bounded, $\sup_t\mathbb{E}\|W_t\|_F^2 \le \mathcal{B}^2$, and the schedule satisfies $\sum_t \eta_t = \infty$ and $\sum_t \eta_t^2 < \infty$. Then Muon with scaled weight decay \eqref{eq:muon-scaled-wd-theory} satisfies
\begin{equation}\label{eq:muon-scaled-stationarity}
    \limsup_{T\to\infty}\ \frac{1}{S_T}\sum_{t=0}^{T-1} \eta_t\, \mathbb{E}\|\nabla f(W_t)\|_*
    \;\le\; 2\sigma\sqrt{\frac{(1-\beta)\,r}{B}},
    \qquad S_T = \sum_{t=0}^{T-1}\eta_t~,
\end{equation}
The bias and drift terms vanish in the limit; the remaining floor is the same noise floor as for Muon without weight decay, since the scaled shrink term enters the descent inequality only at order $\eta_t^2$ and contributes a summable perturbation.
\end{theorem}

The noise floor is due to a fixed batch size and momentum and matches the previous analysis of convergence of Muon without weight decay \cite{shen2025muon,sato2026muon,chang2026convergencemuon}. The noise term vanishes in the deterministic case $\sigma = 0$. In that regime the weighted average of $\mathbb{E}\|\nabla f(W_t)\|_*$ tends to zero, and when $f$ is in addition convex and the iterates are almost surely bounded this gives a best-iterate guarantee in expected function value.

\begin{theorem}\label{thm:scaled-wd-muon-convex}
Suppose the assumptions of Theorem~\ref{thm:scaled-wd-muon-stationarity} hold with $\sigma = 0$, additionally suppose that $\|W_t\|_F\le\mathcal{B}$ almost surely for every $t$, and suppose $f$ is convex with a minimizer $W^\star$ and $f^\star = f(W^\star)$. Then
\begin{equation}\label{eq:muon-scaled-function-gap}
    \frac{1}{S_T}\sum_{t=0}^{T-1}\eta_t\,\mathbb{E}\bigl[f(W_t)-f^\star\bigr]\to0,
    \qquad
    \liminf_{t\to\infty}\mathbb{E}\bigl[f(W_t)-f^\star\bigr]=0~,
\end{equation}
so Muon with scaled weight decay has a best-iterate convergence guarantee for the original objective.
\end{theorem}

Both theorems are proved in Appendix~\ref{app:scaled-wd-muon}. Applying the convex result to the shifted matrix quadratic shows that scaled-decay Muon approaches the true minimizer along a subsequence, in direct contrast to the constant-decay trap of Proposition~\ref{prop:const-wd-muon}.

\begin{corollary}\label{cor:scaled-wd-muon-quadratic}
Let $f(W) = \tfrac{1}{2}\|W - A\|_F^2$ and consider deterministic Muon with scaled weight decay under a Robbins--Monro schedule, assuming the iterates are bounded. Then $\liminf_{t\to\infty}f(W_t)=0$, so there is a subsequence $W_{t_k}\to A=W^\star$, the minimizer of the original objective. In particular scaled-decay Muon is not asymptotically separated from targets outside the radius-$1/\lambda$ ball that traps constant weight decay in Proposition~\ref{prop:const-wd-muon}.
\end{corollary}

\begin{proof}
The objective is $1$-smooth, convex, and bounded below by $0$, with deterministic gradient ($\sigma = 0$, so the noise floor in \eqref{eq:muon-scaled-stationarity} vanishes). Theorem~\ref{thm:scaled-wd-muon-stationarity} gives $\tfrac{1}{S_T}\sum_t \eta_t\|\nabla f(W_t)\|_* \to 0$, and Theorem~\ref{thm:scaled-wd-muon-convex} gives $\liminf_t f(W_t)=0$. Choosing a subsequence along which $f(W_{t_k})\to0$ and using $f(W)=\tfrac12\|W-A\|_F^2$ yields $W_{t_k}\to A$.
\end{proof}

All of these guarantees describe a single trajectory run with a decaying schedule $\eta_t \to 0$, which is where scaled and constant weight decay genuinely differ. The constant-decay perturbation accumulates as $\sum_t \eta_t$ and shifts the fixed point, whereas the scaled-decay perturbation is summable and thus does not. However, this separation does not arise at a fixed learning rate, where $q\eta^2$ and $\lambda\eta$ are both constants and the two updates coincide. Figure~\ref{fig:convergence-illustration} summarizes the full picture, contrasting the scaled and constant weight decay for both SGD and Muon on shifted quadratic objectives. 

\section{Experiment on Training MoE with Muon-SW}
\label{sec:experiments}
To empirically validate the efficacy of scaled weight decay, we train decoder-only MoE transformers with Muon-SW and compare against standard Muon. Our aim is to evaluate scaled weight decay under optimization conditions that mirror frontier pre-training, albeit in a scaled-down version. As summarized in Appendix \ref{app:frontier}, open pre-training recipes have largely converged on a common pattern. They apply Muon to the matrix-shaped Transformer weights, use a weight decay coefficient of $0.1$, decay the learning rate to one tenth of its peak with a cosine schedule, and train for several hundred tokens per active parameter \cite{kimiteam2026kimik2openagentic, 5team2025glm45agenticreasoningcoding, deepseekai2026deepseekv4highlyefficientmilliontoken}. We replicate these optimization choices, fixing $\lambda = 0.1$, using the same cosine schedule, and training at $600$--$650$ tokens per active parameter. The architecture is a LLaMA-style MoE decoder with $8$ experts and top-$2$ routing, trained on the FineWeb dataset tokenized with a GPT-2 tokenizer \cite{radford2019language, touvron2023llama2openfoundation, penedo2024finewebdatasetsdecantingweb}. The global batch consists of $256$ sequences of length $2{,}048$, corresponding to $256\times2{,}048=524{,}288$ tokens per optimizer update. Complete architecture and optimizer settings are given in Appendix \ref{app:arch-hypers}. Muon (and Muon-SW) is applied only to the matrix-shaped weights; the embedding, un-embedding, and other non-matrix parameters are trained with AdamW with no weight decay, and we rescale the Muon learning rate by $0.2\sqrt{\max(m,n)}$ for a weight of shape $m\times n$ so that the same learning rate can be used for both optimizers \cite{liu2025muonscalablellmtraining}.

Since hyperparameter tuning is computationally expensive, we employ $\mu$P-MoE parameterization to transfer the peak learning rate across widths \cite{yang2022tensorprogramsvtuning, malasnicki2025muparametrizationmixtureexperts}. We verify this width transfer directly in Appendix \ref{app:mup-moe} and adopt $\eta = 0.02$ as the transferred peak learning rate. Furthermore, because the optimal peak learning rate is a function of the training horizon, we rescale it by the cube-root rule $\eta^\star(T) = \eta^\star(T_0)(T/T_0)^{-1/3}$, which we also validate in Appendix \ref{app:mup-moe} \cite{bjorck2025scalingoptimallrtoken, ren2026}. Combining the two lets us set the peak learning rate for each width and token budget in Table \ref{tab:main-results} without any per-run tuning, so that the only deliberate change between the Muon and Muon-SW columns is the weight-decay rule itself.

\begin{table}[htb]
\centering
\small
\setlength{\tabcolsep}{4.5pt}
\renewcommand{\arraystretch}{1.15}
\begin{tabular}{@{}rcccccccc@{}}
\toprule
\textbf{Width} &
\textbf{Total} &
\textbf{Active} &
\textbf{Steps} &
\textbf{Tokens} &
\textbf{Tok./Act.} &
\textbf{Peak LR} &
\textbf{Loss (Muon / SW)} &
\textbf{Speed-up} \\
\midrule
256  & 72.7M  & 30.2M  & 35k  & 18.4B  & $\sim 610$ & $1.0{\times}10^{-2}$ & 3.350 / \textbf{3.297} & 21.7\% \\
512  & 264.9M & 95.0M  & 110k & 57.7B  & $\sim 610$ & $7.0{\times}10^{-3}$ & 2.971 / \textbf{2.915} & 27.5\% \\
768  & 520.0M & 180.3M & 223k & 116.9B & $\sim 650$ & $5.6{\times}10^{-3}$ & 2.799 / \textbf{2.754} & 26.5\% \\
1024 & 932.4M & 309.6M & 376k & 197.1B & $\sim 640$ & $4.8{\times}10^{-3}$ & 2.675 / \textbf{2.630} & 29.4\% \\
\bottomrule
\end{tabular}%
\caption{\textbf{Comparison of Muon and Muon-SW for MoE training.} All runs use $8$ experts with top-$2$ routing and are trained at $600$--$650$ tokens per active parameter with a constant true batch of $524{,}288$ tokens per optimizer update. The weight-decay coefficient is held constant at $0.1$ for both methods, and each run uses a linear warmup of $\max\{1000,\, 0.01\,T\}$ steps for total steps $T$ followed by a cosine decay to $1/10$ of the peak learning rate by the end of training. Peak learning rates are muP-transferred across width and then adjusted based on training duration. We report the best validation loss for Muon (standard constant weight decay) and Muon-SW (scaled weight decay); the speed-up is the reduction in optimizer steps for Muon-SW to reach Muon's best validation loss.}
\label{tab:main-results}
\end{table}

The results in Table \ref{tab:main-results} show that Muon-SW improves the best validation loss at every width. The speed-up tends to increase with scale, although it is not strictly monotonic across the evaluated widths, and is largest at width $1024$. Measured as the reduction in optimizer steps needed for Muon-SW to reach Muon's best loss, the speed-up rises from $21.7\%$ at width $256$ to $29.4\%$ at width $1024$. While our largest run is orders of magnitude below true frontier scale, the generally larger benefit at greater model sizes and longer training durations suggests that scaled weight decay could substantially accelerate frontier model training if the trend continues.

\section{Weight Norms under Constant and Scaled Weight Decay}
\label{sec:steady-state}

We now analyze how the size of the weights evolves over training under the two decay rules. Consider a single matrix-valued weight $W_t \in \mathbb{R}^{m\times n}$ updated by Muon. Writing $\eta_t$ for the base learning rate at step $t$, $\widehat{\eta}_t = 0.2\,\eta_t\sqrt{\max(m,n)}$ for the actual matrix-scaled Muon update magnitude, and $O_t = \mathrm{polar}(M_t)$ for the orthogonalized momentum direction, the Muon update with decoupled weight decay is
\begin{equation}\label{eq:ss-muon-update}
    W_{t+1} = (1 - \varphi_t)\,W_t - \widehat{\eta}_t\, O_t~,
\end{equation}
where the factor $0.2\sqrt{\max(m,n)}$ rescales the base learning rate so that the polar transformed update has a root-mean-square (RMS) scale comparable to an AdamW update \cite{liu2025muonscalablellmtraining}, and $\varphi_t$ is the decay factor computed from the base learning rate: $\varphi_t = \lambda\eta_t$ for constant weight decay and $\varphi_t = \lambda\eta_t^2/\etamax$ for scaled weight decay. In particular, Muon-SW uses
\[
    W_{t+1} = \left(1-\lambda\frac{\eta_t^2}{\etamax}\right)W_t - \widehat{\eta}_t O_t~.
\]

The quantity we track is the RMS norm of the weight,
\begin{equation}\label{eq:ss-rms-def}
    r_t = \frac{\|W_t\|_F}{\sqrt{mn}}~,
\end{equation}
which measures the typical magnitude of an individual entry of $W_t$ and is the natural scale-invariant summary of how large the weights are.

To see how $r_t$ evolves, we expand the squared Frobenius norm of the update \eqref{eq:ss-muon-update},
\begin{equation}\label{eq:ss-norm-expand}
    \|W_{t+1}\|_F^2
    = (1-\varphi_t)^2 \|W_t\|_F^2
    - 2(1-\varphi_t)\,(0.2\,\eta_t\sqrt{\max(m,n)}) \langle W_t, O_t\rangle
    + (0.2\,\eta_t)^2 \max(m,n)\, \|O_t\|_F^2~,
\end{equation}
where $\langle \cdot,\cdot\rangle$ is the Frobenius inner product. Because $O_t = \mathrm{polar}(M_t)$ has all nonzero singular values equal to $1$, its squared Frobenius norm equals its rank. In the generic case the momentum buffer is full rank, so $\|O_t\|_F^2 = \min(m,n)$ and hence $\max(m,n)\,\|O_t\|_F^2 = mn$. To capture the cross term, we introduce the normalized alignment between the update direction and the current weight,
\begin{equation}\label{eq:ss-alignment}
    a_t = \frac{\langle W_t, O_t\rangle}{\|W_t\|_F\,\|O_t\|_F} \in [-1,1]~,
\end{equation}
which is the cosine of the angle between $O_t$ and $W_t$. For the weight matrices that feed directly into a normalization layer, the loss is invariant to the scale of the rows (or columns) of $W_t$, and this scale-invariance forces the gradient to be orthogonal to the weight, $\langle W_t, \nabla f(W_t)\rangle = 0$. We verify this orthogonality in the left panel of Figure~\ref{fig:app-gm-alignment}. This orthogonality between weights and the gradient is exploited in the analysis of \cite{vanlaarhoven2017, kosson2024, defazio2025}.

Since the momentum buffer $M_t$ is an exponential average of such gradients, it does not inherit this instantaneous orthogonality property (right panel of Figure~\ref{fig:app-gm-alignment}), and neither does its polar transformation $O_t = \mathrm{polar}(M_t)$. Empirically the alignment $a_t$ is negative, meaning the update pushes the weight outward rather than inward, as displayed in the left panel of Figure~\ref{fig:w1024-align}. Using $\|O_t\|_F = \sqrt{\min(m,n)}$, the cross term becomes $2(1-\varphi_t)(0.2\,\eta_t)\,a_t\,\|W_t\|_F\sqrt{mn}$; substituting into \eqref{eq:ss-norm-expand} and dividing by $mn$, the shape factors cancel and the RMS norm obeys the clean recursion
\begin{equation}\label{eq:ss-rms-recursion}
    r_{t+1}^2
    = (1-\varphi_t)^2\, r_t^2
    - 2(1-\varphi_t)\,(0.2\,\eta_t)\, a_t\, r_t
    + (0.2\,\eta_t)^2~.
\end{equation}

The base coefficient $\lambda\eta_t$ is small (of order $10^{-3}$ or below at the learning rates we use), and the scaled coefficient $\lambda\eta_t^2/\etamax$ is smaller still by the factor $\eta_t/\etamax \le 1$. We may therefore drop terms of order $\varphi_t^2$ and $\varphi_t\eta_t$ wherever they appear against the leading contributions. Thus, the recursion \eqref{eq:ss-rms-recursion} simplifies to
\begin{equation}\label{eq:ss-rms-recursion-simplified}
    r_{t+1}^2
    \approx (1 - 2\varphi_t)\, r_t^2
    - 2\,(0.2\,\eta_t)\, a_t\, r_t
    + (0.2\,\eta_t)^2~.
\end{equation}

\begin{figure}[htb]
    \centering
    \includegraphics[width=0.9\linewidth]{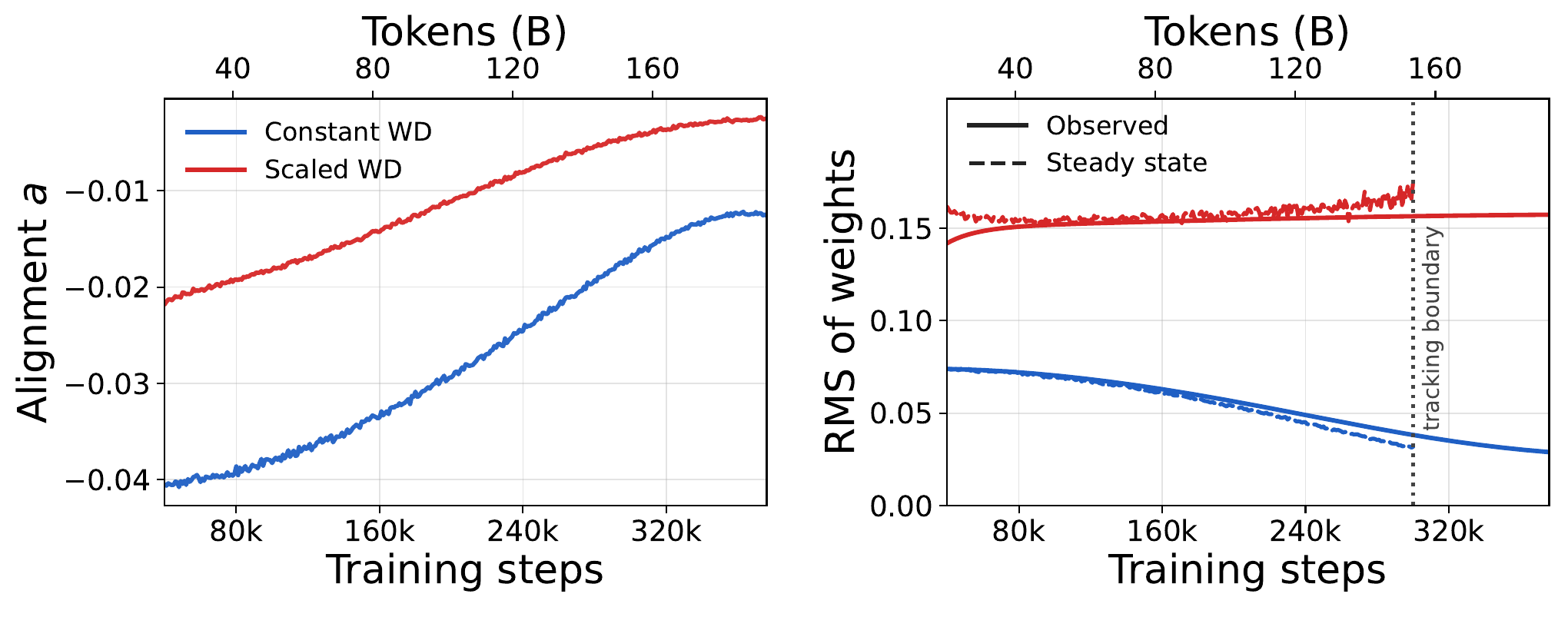}
    \caption{\textbf{Weight--update alignment and steady-state weight norm (width 1024).} \emph{Left:} the aggregate alignment $a_t$ between the weights and the Muon update is negative under both rules, with its magnitude decreasing through training as the learning rate and relative update size fall. \emph{Right:} the observed RMS weight norm (solid) against the instantaneous steady-state prediction (dashed). The theory captures the separation between the two, with scaled weight decay sustaining a nearly constant norm, whereas constant weight decay produces a norm that falls by roughly $60\%$ from its peak value. The steady-state approximation is predictive before the marked tracking boundary near $300$k steps; beyond it the remaining relaxation time is too short for the norm to follow its moving target.}
    \label{fig:w1024-align}
\end{figure}

At each step the norm is pulled toward the value at which it would stop changing, $r_{t+1} = r_t = r^\star_t$. This is a \emph{quasi-steady} target with the coefficients $\eta_t$, $\varphi_t$, and $a_t$ drifting slowly with the schedule, so rather than a single fixed point we obtain a time-dependent target $r^\star_t$ that the norm chases. Setting $r_{t+1} = r_t = r^\star_t$ in \eqref{eq:ss-rms-recursion-simplified} and writing $a_t = -|a_t|$ (the alignment is negative) gives a quadratic in $r^\star_t$,
\begin{equation}\label{eq:ss-quadratic}
    2\varphi_t\, (r^\star_t)^2 - 2\,(0.2\,\eta_t)\,|a_t|\, r^\star_t - (0.2\,\eta_t)^2 = 0~,
\end{equation}
whose positive root is
\begin{equation}\label{eq:ss-rstar}
    r^\star_t = \frac{0.2\,\eta_t}{2\varphi_t}\Bigl(|a_t| + \sqrt{a_t^2 + 2\varphi_t}\Bigr)~.
\end{equation}
% Substituting the two decay coefficients makes the dependence on $\lambda$, $\eta_t$, and $a_t$ explicit. For constant weight decay ($\varphi_t = \lambda\eta_t$) the leading factor of $\eta_t$ cancels,
% \begin{equation}\label{eq:ss-rstar-const-explicit}
%     r^\star_t = \frac{0.1}{\lambda}\Bigl(|a_t| + \sqrt{a_t^2 + 2\lambda\eta_t}\Bigr)~,
% \end{equation}
% whereas for scaled weight decay ($\varphi_t = \lambda\eta_t^2/\etamax$) one factor of $\eta_t$ survives in the denominator,
% \begin{equation}\label{eq:ss-rstar-scaled-explicit}
%     r^\star_t = \frac{0.1\,\etamax}{\lambda\eta_t}\Bigl(|a_t| + \sqrt{a_t^2 + 2\lambda\eta_t^2/\etamax}\Bigr)~.
% \end{equation}

The quasi-steady target \eqref{eq:ss-rstar} balances the decay towards the origin against the outward push of the aligned update. Note that for Muon the polar transformation is applied by five Newton--Schulz steps, which is not sufficient to produce a perfectly orthogonal factor: as shown in Figure~\ref{fig:app-ns5}, the applied update has normalized Frobenius norm close to $0.88$ rather than the ideal value of one, so its RMS is smaller than $0.2\,\eta_t$. We therefore evaluate \eqref{eq:ss-rstar} with the measured update magnitude in place of $0.2\,\eta_t$. The resulting target $r^\star_t$ is the dashed curve in the right panel of Figure~\ref{fig:w1024-align}; it tracks the observed norm across the bulk of training and separates the two decay rules sharply.

The quasi-steady description is accurate only while the norm relaxes toward the target faster than the target itself moves. Freezing the slowly varying coefficients, the recursion \eqref{eq:ss-rms-recursion-simplified} contracts toward $r^\star_t$ with multiplier $1 - 2\varphi_t$ per step, so the norm relaxes on a timescale
\begin{equation}\label{eq:ss-tau}
    \tau_t = \frac{1}{2\varphi_t}~,
\end{equation}
and adiabatic tracking requires
\begin{equation}\label{eq:ss-adiabatic}
    \tau_t\,\Bigl\lvert \frac{\mathrm{d}\ln r^\star_t}{\mathrm{d}t}\Bigr\rvert \ll 1~.
\end{equation}
Since $\varphi_t$ is proportional to the learning rate ($\lambda\eta_t$ for constant decay and $\lambda\eta_t^2/\etamax$ for scaled decay), the relaxation time $\tau_t$ lengthens as the schedule decays, and for scaled decay it grows as $\eta_t^{-2}$. In the width-1024 run condition \eqref{eq:ss-adiabatic} holds comfortably through the bulk of training and is violated only in the final decay tail, near $300$k steps, which is the boundary marked in Figure~\ref{fig:w1024-align}. The observed departure from \eqref{eq:ss-rstar} in the tail is a breakdown of adiabatic tracking, not of the recursion relation \eqref{eq:ss-rms-recursion-simplified}.

The two decay rules produce qualitatively different targets determined by how the alignment changes with the learning rate. As shown in Figure~\ref{fig:app-late-alignment}, under scaled weight decay the alignment is proportional to the learning rate, $a_t \approx -k\,\eta_t$ with $k \approx 20.8$, whereas under constant weight decay it retains a nonzero offset, $a_t \approx -a_0 - k'\eta_t$ with a small coefficient $a_0 \approx 7.4\times10^{-3}$, so that $\lvert a_t\rvert$ approaches a constant as $\eta_t \to 0$. Substituting these fits into the full quasi-steady target \eqref{eq:ss-rstar} determines the end-of-training norm as $\eta_t$ reaches its floor. For scaled weight decay every term carries a factor $\eta_t^2$ that cancels, leaving a value independent of the learning rate,
\begin{equation}\label{eq:ss-rstar-scaled}
    r^\star_t \;\xrightarrow{}\; \frac{0.2\,\etamax}{2\lambda}\Bigl(k + \sqrt{k^2 + 2\lambda/\etamax}\Bigr)~,
\end{equation}
so the norm settles onto a high plateau. For constant weight decay the same limit is set by the alignment offset $a_0$,
\begin{equation}\label{eq:ss-rstar-constant}
    r^\star_t \;\xrightarrow{\eta_t\to 0}\; \frac{0.2\,a_0}{\lambda}~,
\end{equation}
reached from above as the magnitude of the alignment decreases. The two limits differ by more than an order of magnitude: with $\lambda = 0.1$ and $\etamax = 4.8\times10^{-3}$, \eqref{eq:ss-rstar-scaled} evaluates to $r^\star_t \approx 0.20$ while \eqref{eq:ss-rstar-constant} gives $r^\star_t \approx 0.015$. In other words, constant weight decay drives the weights toward an operating norm roughly $14\times$ smaller than scaled weight decay at the same peak learning rate and decay coefficient. These are the strict $\eta_t \to 0$ limits, but the collapse is already clear in Figure~\ref{fig:w1024-align}, where by the final learning rate $\eta_t = \etamax/10$ the norm under constant weight decay has fallen roughly $60\%$ from its post-warmup peak while the norm is essentially flat for scaled weight decay. At this final learning rate the two runs already differ by a factor of about $5.4$ ($r \approx 0.16$ against $r \approx 0.029$), and the same contrast is demonstrated at widths $512$ and $768$ in Figures~\ref{fig:app-w512} and~\ref{fig:app-w768}.

\section{Related Work}
\label{sec:related-work}

Weight decay for improving generalization by imposing an $\ell_2$ penalty was first proposed by Krogh and Hertz \cite{krogh1991}. The importance of decoupling the weight decay from the gradient update for adaptive optimizers was realized by Loshchilov and Hutter \cite{loshchilov2019decoupledweightdecayregularization}, and in the process they proposed AdamW which has been widely adopted for training deep neural networks. In \cite{dangelo2024needweightdecaymodern}, the authors explained how weight decay balances the bias-variance tradeoff, improving generalization and training stability for large language models. Variants on decoupled weight decay have been proposed such as cautious weight decay \cite{chen2026cautiousweightdecay}, and adjusting it based on the gradient norm \cite{xie2022understanding}. Entirely removing weight decay, and optimizing the network by restricting the weights to the hypersphere was suggested in \cite{ren2026,wen2026}. Due to the inclusion of LayerNorm in transformers, it is possible that good performance can be achieved by normalizing the weights of the network but this method has not been widely adopted and it requires further investigation. The steady-state dynamics of the weight norm under weight decay was analyzed by van Laarhoven \cite{vanlaarhoven2017}, extended to Adam and Lion in \cite{kosson2024}, and has also been applied to schedule-free spectral optimization \cite{apte2026anytimetrainingschedulefreespectral}.

The closest reference to our work is the paper by Defazio \cite{defazio2025}, which proposes the same correction, scaling weight decay by the fraction of peak learning rate $\eta_t/\eta_{\max}$, to combat the rise in gradient norms late in training. Our manuscript extends his results in three substantial ways. Firstly, his steady-state argument relies on the varying gradient norm and the orthogonality of gradient to normalized-layer weights, so it does not apply to optimizers like Muon with updates which are normalized (approximately) and not orthogonal to weight matrices. Secondly, we motivate the same rule through Robbins--Monro summability and establish that it preserves the asymptotic stationarity guarantee associated with the original unregularized loss. Thirdly, we undertake a systematic evaluation of the method on large-scale models across different model sizes and training horizons and demonstrate its effectiveness. With frontier training now shifting primarily to Muon over AdamW, our method is of immediate practical relevance. Independent evidence for the practical relevance of this rule comes from Inkling, a frontier-scale model trained by Thinking Machines Lab using Muon with a decay step $\propto \lambda\eta_t^2$ \cite{inkling2026}. They report that this choice kept the overall size of the model weights stable across training horizons, consistent with our results in Figures~\ref{fig:w1024-main} and~\ref{fig:w1024-align}.

\section{Conclusion and Future Work}
\label{sec:conclusion}

Drawing on the Robbins--Monro conditions \cite{Robbins1951}, we proposed scaled weight decay, which multiplies the decoupled decay coefficient by the ratio $\eta_t/\etamax$. Applied to Muon, the resulting Muon-SW optimizer attains the Muon validation loss using $30\%$ fewer training tokens at our largest scale. The speed-up generally increases with scale and is largest at width $1024$. At frontier scale these savings would lead to a sizable reduction in training cost.

Alongside these empirical results, we proved that scaled weight decay retains the asymptotic stationarity guarantees of unregularized SGD and Muon. In contrast, we showed that constant weight decay can remain bounded away from a minimizer on simple quadratic objectives. We also developed a steady-state analysis of the weight norm that shows why constant weight decay causes a substantial decrease in the weight norm, while it remains roughly constant under scaled weight decay. Taken together, these results show that scaled weight decay preserves the optimization benefits of conventional weight decay while avoiding the additional asymptotic bias induced by constant decay, requiring only a one-line modification to existing training pipelines.

There are several directions for future work. On the architecture side, it is important to test whether the same benefits carry over to other models such as Mamba \cite{gu2024mambalineartimesequencemodeling, dao2024transformersssmsgeneralizedmodels}, Samba \cite{ren2025sambasimplehybridstate}, Gated DeltaNet \cite{yang2025gateddeltanetworksimproving}, and Parallax \cite{zuo2026parallaxparameterizedlocallinear}. On the optimizer side, the scaling principle should apply to other optimizers for matrix-valued weights such as Shampoo \cite{gupta2018, anil2021}, SOAP \cite{vyas2025} and Aurora \cite{dewulf2026auroraleverageawarespectraloptimizer}. Confirming this empirically is a natural next step. It will also be fruitful to validate the method for training models designed for other domains such as vision and audio. Finally, it is important to study the impact of pre-training with scaled weight decay in the post-training of large language models including supervised fine-tuning \cite{ouyang2022traininglanguagemodelsfollow}, preference alignment (reinforcement learning from human feedback \cite{christiano2023deepreinforcementlearninghuman}, direct preference optimization \cite{rafailov2024directpreferenceoptimizationlanguage}) and reinforcement learning for reasoning \cite{shao2024deepseekmathpushinglimitsmathematical, Guo_2025}. 

\section*{Acknowledgments}
We are grateful to Pranav Deshpande and Pragna Subrahmanya for assistance with computing infrastructure. We thank Rob Otter for the executive support of the work, and our colleagues at the Global Technology Applied Research center of JPMorganChase
for support.

\section*{Disclaimer}
This paper was prepared for informational purposes with contributions from the Global Technology Applied Research center of JPMorgan Chase \& Co. This paper is not a product of the Research Department of JPMorgan Chase \& Co. or its affiliates. Neither JPMorgan Chase \& Co. nor any of its affiliates makes any explicit or implied representation or warranty and none of them accept any liability in connection with this paper, including, without limitation, with respect to the completeness, accuracy, or reliability of the information contained herein and the potential legal, compliance, tax, or accounting effects thereof. This document is not intended as investment research or investment advice, or as a recommendation, offer, or solicitation for the purchase or sale of any security, financial instrument, financial product or service, or to be used in any way for evaluating the merits of participating in any transaction.

\newpage 
\printbibliography

\newpage
\appendix

\noindent\makebox[\linewidth]{\rule{\linewidth}{1.5pt}}
\section*{\centering \large Technical Appendices and Supplementary Material}
\noindent\makebox[\linewidth]{\rule{\linewidth}{1.5pt}}

In this appendix, we provide additional background, experimental details, proofs, and implementation details that supplement the main text. This appendix is organized as follows:
\begin{itemize}
    \item In \cref{app:frontier}, we review the training recipes of recent open frontier Mixture-of-Experts models and in particular the number of tokens used in training per active model parameter.
    \item In \cref{app:arch-hypers}, we describe the MoE architecture and the hyperparameter choices used across all experiments.
    \item In \cref{app:mup-moe}, we review $\mu$P for MoE models and the learning-rate transfer rules used to tune learning rate across width and training horizon.
    \item In \cref{app:proof-details}, we give detailed proofs deferred from the main text, including the non-convergence of constant weight decay and convergence of scaled weight decay. 
    \item In \cref{app:additional-data}, we present additional data from the MoE training experiments. 
    \item In \cref{app:implementation}, we provide a PyTorch implementation of the Muon-SW algorithm with scaled weight decay.
\end{itemize}

\section{Training Details of Open Frontier Models}
\label{app:frontier}

Open frontier Mixture-of-Experts (MoE) models released since 2025 with disclosed training recipes have adopted Muon as the main optimizer over AdamW \cite{kimiteam2026kimik2openagentic, 5team2025glm45agenticreasoningcoding, deepseekai2026deepseekv4highlyefficientmilliontoken}. AdamW is still retained as the optimizer for selected parameter groups such as embeddings, normalization weights, or output heads. This combination improves optimization efficiency and final model quality relative to AdamW baselines. 
Rescaling the learning rate of Muon based on the shape of the parameters, allows one to use the same learning rate for both the optimizers when employed in conjunction.

To make the comparison concrete, Table \ref{tab:frontier-moe-training} collects the parts of these pre-training recipes that are most relevant for this paper: the active model size, the token horizon, the learning-rate schedule, and the Muon weight-decay value. The striking point is that the recipes have converged on a similar optimizer pattern: Muon is used for matrix-shaped Transformer weights, the nominal weight decay is around $0.1$, the cosine schedule ends at roughly one tenth of the peak learning rate, and the training horizon is several hundred tokens per active parameter. We replicate this qualitative regime in our experiments, which are more modest in scale but still allow us to study how changing from constant to scaled weight decay affects the optimization dynamics and speed of convergence.

\begin{table}[htb]
\centering
\small
\setlength{\tabcolsep}{4.5pt}
\renewcommand{\arraystretch}{1.15}
\begin{tabular}{@{}lccccccc@{}}
\toprule
\textbf{Model} &
\textbf{Total} &
\textbf{Active} &
\textbf{Tokens} &
\textbf{Peak LR} &
\textbf{Final LR} &
\textbf{WD} &
\textbf{Tok./Act.} \\
\midrule
GLM-4.5
& 355B
& 32B
& 22T
& $2.5{\times}10^{-4}$
& $2.5{\times}10^{-5}$
& 0.1
& $\sim 700$ \\

Kimi K2
& 1T
& 32B
& 15.5T
& $2.0{\times}10^{-4}$
& $2.0{\times}10^{-5}$
& 0.1
& $\sim 480$ \\

DeepSeek-V4-Pro
& 1.6T
& 49B
& 32T
& $2.0{\times}10^{-4}$
& $2.0{\times}10^{-5}$
& 0.1
& $\sim 650$ \\
\bottomrule
\end{tabular}
\caption{Model and pre-training details for recent open frontier MoE models. Total and active parameter counts are reported in parameters, tokens are reported in training tokens, and Tok./Act. is computed by dividing reported training tokens by activated parameters. The weight decay value is for matrix-shaped parameters trained with Muon, excluding embedding and unembedding layers.}
\label{tab:frontier-moe-training}
\end{table}

A second notable observation is that these models are trained on far more tokens per active parameter than the dense Chinchilla heuristic of roughly 20 tokens per parameter \cite{hoffmann2022trainingcomputeoptimallargelanguage}. For MoE models, this comparison should be made against active rather than total parameters, since training and inference FLOPs are governed primarily by the routed subset of parameters. Scaling laws for MoE models indicate that increasing the number of experts can itself shift the compute-optimal allocation toward fewer active parameters and more training tokens, raising the optimal tokens-per-active-parameter ratio relative to dense models \cite{krajewski2024scalinglawsfinegrainedmixture, tian2025greaterleveragescalinglaws}. Even taking this into account, the observed ratios are much larger than compute-optimal scaling would suggest.

Frontier MoE model development is not only optimized for pre-training compute efficiency, but must also account for deployment efficiency and serving cost. Since extra pre-training cost is paid once, while active-parameter inference cost is paid on every token served, it is economically rational to train a smaller active model on substantially more data. Thus, the very high tokens-per-active-parameter ratios observed in recent frontier MoEs likely reflect both MoE scaling behavior and lifetime-cost optimization \cite{erdil2025inferenceeconomicslanguagemodels}. 

\section{Architecture and Hyperparameter Details}
\label{app:arch-hypers}

All experiments use the same $\mu$P LLaMA-style \cite{touvron2023llama2openfoundation} sparse MoE decoder, trained on FineWeb \cite{penedo2024finewebdatasetsdecantingweb} with GPT-2 tokenization \cite{radford2019language}. Each model has 12 layers, head dimension 64, 8 experts per layer, and top-2 routing, so each token activates two experts per MoE layer. Each block is pre-norm, with RMS-Norm before attention and before the MoE MLP, plus a final RMSNorm before the tied output head \cite{zhang2019rootmeansquarelayer}. We use RoPE with base $\theta=10000$ \cite{su2023roformerenhancedtransformerrotary}, bias-free linear layers, no dropout, a SiLU-gated SwiGLU expert MLP \cite{shazeer2020gluvariantsimprovetransformer}, and standard MoE routing with \texttt{topk\_softmax} \cite{fedus2022}. Router regularization uses load-balancing loss with coefficient $0.1$ and z-loss with coefficient $0.01$. The experiments were carried out using the \texttt{llm-optimizer-benchmark} codebase \cite{semenov2025benchmarking}. 

\begin{table}[htb]
\centering
\small
\begin{tabular}{lrrrrr}
\toprule
\textbf{Detail} & \textbf{W128} & \textbf{W256} & \textbf{W512} & \textbf{W768} & \textbf{W1024} \\
\midrule
Total params & 26.1M & 72.7M & 264.9M & 520.0M & 932.4M \\
Active params & 12.0M & 30.2M & 95.0M & 180.3M & 309.6M \\
Attention heads & 2 & 4 & 8 & 12 & 16 \\
Expert FF dim & 512 & 768 & 1536 & 2048 & 2816 \\
\bottomrule
\end{tabular}
\caption{\textbf{Model architecture across widths.} The active parameters count the tied embedding \& un-embedding head, attention, routers, norms, and the top-2 selected experts in each layer.}
\label{tab:mup-moe-architecture-counts}
\end{table}

The shared training and optimizer settings are summarized in Table~\ref{tab:training-hyperparameters}. The $\mu$P parameter grouping uses the base learning rate for routers, the width-scaled learning rate $\eta/(d/256)$ for hidden and expert matrices, and zero weight decay on normalization and tied embedding/head parameters. For Muon matrix groups, we use the adjustment factor $0.2\eta_t\sqrt{\max{(m,n)}}$ which lets Muon use the same nominal learning rate as AdamW \cite{liu2025muonscalablellmtraining}. When running Muon, embedding, un-embedding, and non-matrix parameters are trained with AdamW at the identical nominal learning rate, with no weight decay applied to those groups. 

\begin{table}[htb]
\centering
\small
\begin{tabular}{ll}
\toprule
\textbf{Detail} & \textbf{Value} \\
\midrule
Data / tokenizer & FineWeb / GPT-2, vocabulary 50,304 \\
Sequence length & 2048 \\
Hardware & DDP on 8$\times$ A100 GPUs \\
Batch size & 524,288 tokens/update \\
AdamW betas & $(\beta_1,\beta_2)=(0.9,0.95)$ \\
Muon settings & $\beta=0.95$, 5 Newton--Schulz steps \\
Weight decay & $\lambda=0.1$; standard $\lambda\eta_t$, scaled $\lambda\eta_t^2/\eta_{\max}$ \\
Schedule & cosine One-Cycle with final LR $=0.1\peaklr$ \\
Warmup & $\max\{1000,0.01T\}$ steps for total steps $T$ \\
\bottomrule
\end{tabular}
\caption{\textbf{Training and optimization details.} These are shared training settings used across all the runs. The learning rate is adjusted based on the total training duration.}
\label{tab:training-hyperparameters}
\end{table}

\section{$\mu$P for MoE and Learning Rate Transfer Across Horizons}
\label{app:mup-moe}

A key problem in the training of large models is the difficulty of hyperparameter tuning. The $\mu$P parameterization, or maximal update parameterization, is designed to address this problem by enabling the transfer of hyperparameters across model width. This is particularly relevant for frontier models, which have a large number of parameters and can typically only be trained once at the largest scale. By using $\mu$P, one can tune hyperparameters on smaller models and then transfer those settings to larger models, saving computational resources \cite{yang2022tensorprogramsvtuning}.

\begin{table}[htb]
\centering
\begin{tabular}{lccc}
\hline
\textbf{Component} & \textbf{Init. Var.} & \textbf{Multiplier} & \textbf{LR (Adam)} \\
\hline
Embedding & $1$ & $1$ & $1$ \\
Unembedding & $1$ & $1/d_{\mathrm{in}}$ & $1$ \\
Attention $(Q,K,V,O)$ & $1/d_{\mathrm{in}}$ & $1$ & $1/d_{\mathrm{in}}$ \\
Feed-forward (dense) & $1/d_{\mathrm{in}}$ & $1$ & $1/d_{\mathrm{in}}$ \\
\hline
Experts (MoE) & $1/d_{\mathrm{in}}$ & $1$ & $1/d_{\mathrm{in}}$ \\
Router (MoE) & $1$ & $1/d_{\mathrm{in}}$ & $1$ \\
\hline
\end{tabular}
\caption{Recommended $\mu$P scaling rules for dense and MoE Transformers. Here $d_{\mathrm{in}}$ denotes the input dimension of the corresponding weight matrix.}
\label{tab:mup_moe_rules}
\end{table}

The fundamental idea behind $\mu$P is to modify the hyperparameters in a manner that preserves nontrivial feature-learning dynamics as width increases \cite{yang2024spectralconditionfeaturelearning}. This is achieved by scaling the learning rate, variance at initialization, and output multipliers of certain layers so that the training dynamics remain comparable across widths. This idea can be extended to MoE models, where router and expert weights are treated differently in terms of their learning rates and scaling \cite{malasnicki2025muparametrizationmixtureexperts}. For example, router weights are trained at the base learning rate, while the learning rate for expert weights is scaled inversely with width. Table \ref{tab:mup_moe_rules} summarizes the parameterization for dense and MoE Transformers. The listed $\mu$P-MoE rules are derived for width scaling with a fixed number of experts and fixed top-$k$ routing; experiments suggest that it is also possible to transfer hyper-parameters across model depth \cite{yang2023tensorprogramsvifeature}, MoE granularity or the number of experts \cite{ren2026rethinkinglanguagemodelscaling}.

\begin{figure}[htb]
    \centering
    \includegraphics[width=0.6\linewidth]{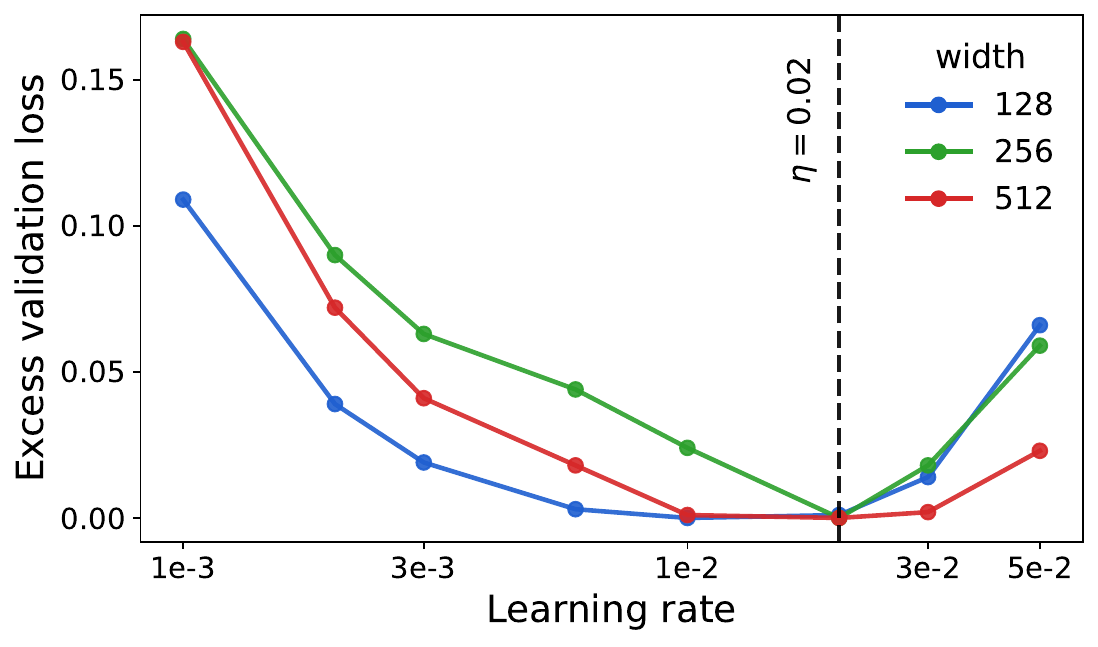}
    \caption{\textbf{$\mu$P-MoE learning-rate transfer across width.}
    We train $\mu$P LLaMA-style sparse MoE models at widths 128, 256, and 512 with fixed MoE granularity. The y-axis shows excess validation loss relative to the best learning rate at each width, making the transfer shape comparable across model sizes. A broad shared low-loss basin appears around $10^{-2}$--$3\times 10^{-2}$, and we select $\eta=0.02$ for subsequent experiments.}
    \label{fig:mup-moe-lr-transfer}
\end{figure}

To test whether the $\mu$P-MoE parameterization gives useful learning-rate transfer across width, we ran a controlled width sweep at fixed MoE granularity. We trained the same $\mu$P LLaMA-style sparse MoE architecture at widths 128, 256, and 512, keeping head dimension fixed by using 2, 4, and 8 attention heads respectively. We swept the peak AdamW learning rate and measured final validation loss after 5k optimizer steps. The results are shown in Figure \ref{fig:mup-moe-lr-transfer} with $\eta=0.02$ lying at or near the minimum for all three widths. We therefore use it as the transferred peak learning rate for subsequent experiments, before any adjustment due to total training duration.

% To test whether the $\mu$P-MoE parameterization gives useful learning-rate transfer across width, we ran a controlled width sweep at fixed MoE granularity. We trained the same $\mu$P LLaMA-style sparse MoE architecture at widths 128, 256, and 512, keeping head dimension fixed by using 2, 4, and 8 attention heads respectively. We swept the peak AdamW learning rate and measured final validation loss after 5k optimizer steps. The results are shown in Figure \ref{fig:mup-moe-lr-transfer}, and also tabulated in Table \ref{tab:mup_moe_lr_sweep}. 

% \begin{table}[htb]
% \centering
% \begin{tabular}{lccc}
% \toprule
% \textbf{Peak LR} & \textbf{Width 128} & \textbf{Width 256} & \textbf{Width 512} \\
% \midrule
% $1\times 10^{-3}$ & 3.948 & 3.667 & 3.451 \\
% $2\times 10^{-3}$ & 3.877 & 3.594 & 3.359 \\
% $3\times 10^{-3}$ & 3.856 & 3.565 & 3.327 \\
% $6\times 10^{-3}$ & 3.840 & 3.546 & 3.304 \\
% $1\times 10^{-2}$ & \textbf{3.835} & 3.524 & 3.287 \\
% $2\times 10^{-2}$ & \textbf{3.835} & \textbf{3.500} & \textbf{3.286} \\
% $3\times 10^{-2}$ & 3.845 & 3.516 & \textbf{3.286} \\
% $5\times 10^{-2}$ & 3.896 & 3.555 & 3.306 \\
% \bottomrule
% \end{tabular}
% \caption{$\mu$P-MoE learning-rate transfer sweep at widths 128, 256, and 512 with fixed MoE granularity: 8 experts per layer and top-2 routing. Final validation loss is reported after 5k optimizer steps. Bold values denote the best validation loss within each width, including ties.}
% \label{tab:mup_moe_lr_sweep}
% \end{table}

\begin{figure}[htb]
    \centering
    \includegraphics[width=0.6\linewidth]{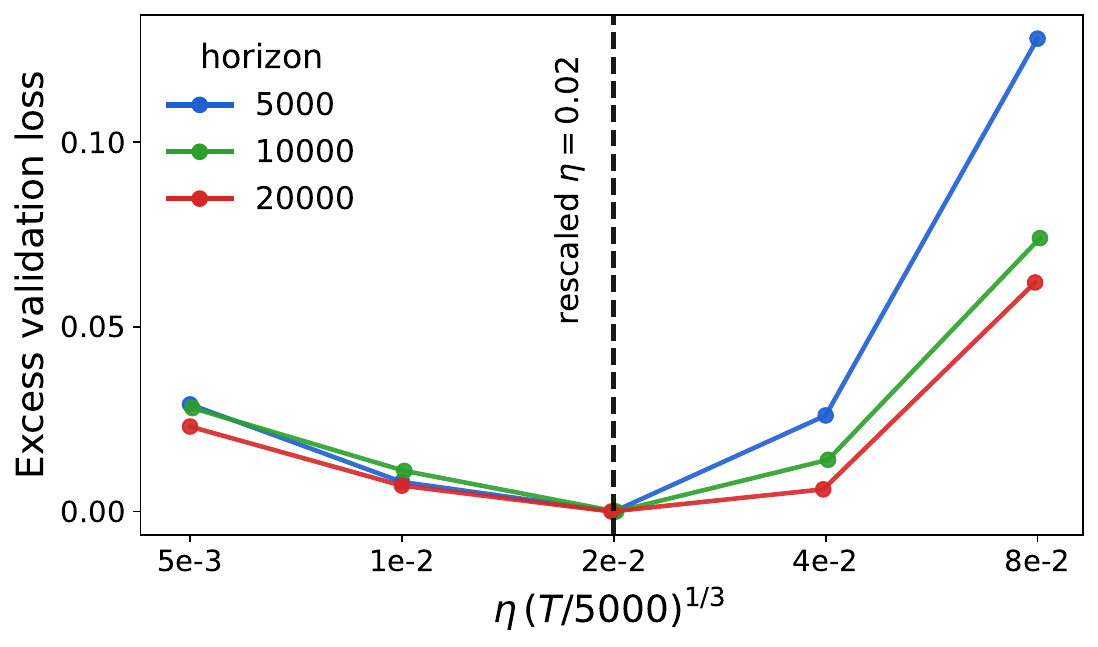}
    \caption{\textbf{Learning-rate transfer across training horizon.}
    We train the $\mu$P LLaMA-style sparse MoE model at width $256$ for horizons $T\in\{5000, 10000, 20000\}$ optimizer steps and sweep the peak AdamW learning rate. The y-axis shows excess validation loss relative to the best learning rate at each horizon; the x-axis is the horizon-rescaled learning rate $\eta\,(T/5000)^{1/3}$. Under this rescaling the three curves share a common minimum near $0.02$ (dashed line), so the peak learning rate for a new horizon can be obtained by rescaling the width-transferred value $\eta=0.02$.}
    \label{fig:mup-moe-horizon-transfer}
\end{figure}

Although $\mu$P transfers the peak learning rate across width, the optimal peak learning rate still depends on the training duration, with longer runs favoring smaller peak learning rates. The optimal learning rate for longer horizons can be computed by tuning the learning rate for a shorter horizon and transferring it using the law suggested in \cite{bjorck2025scalingoptimallrtoken, ren2026}:
\begin{equation}\label{eq:lr-horizon-law}
    \eta^\star(T) = \eta^\star(T_0)\,\bigg(\frac{T}{T_0}\bigg)^{-1/3}~,
\end{equation}
where $T_0$ is a reference horizon and $\eta^\star(T)$ is the optimal peak learning rate at horizon $T$.

We test this for our training setup by fixing the width at $256$ and sweeping the peak AdamW learning rate at three training horizons, $T \in \{5000, 10000, 20000\}$ optimizer steps. We find that the optima are well aligned once the learning rate is rescaled by $(T/5000)^{1/3}$ as illustrated in Figure \ref{fig:mup-moe-horizon-transfer}. This cube-root rule lets us set the peak learning rate for a target horizon by adjusting the width-transferred value of $\eta=0.02$. In combination, $\mu$P width transfer and the cube-root horizon rule let us set the peak learning rate for any width and training horizon directly, with no additional tuning needed.

\section{Proofs for Convergence Theory}
\label{app:proof-details}

This appendix collects the detailed proofs deferred from Section \ref{sec:convergenc-theory}. Throughout, the objective being optimized is the original unregularized training objective, so when we say that an algorithm converges the reference point is a minimizer (or stationary point) of that objective.

\subsection{Non-convergence of constant weight decay on a shifted quadratic}
\label{app:const-wd-1d}

We restate and prove in full the one-dimensional non-convergence statement given in the main text as Proposition \ref{prop:const-wd-nonconvergence}. The setup is deterministic (noise-free), so no stochastic assumptions are needed. Even in the absence of gradient noise the iteration fails to reach the minimizer.

\begin{proposition}[restatement of Proposition \ref{prop:const-wd-nonconvergence}]
\label{prop:app-const-wd-1d}
Let $a>0$ and $c\neq0$, and consider the deterministic quadratic $f(w)=\tfrac{a}{2}(w-c)^2$ with unique minimizer $w^\star=c$. Fix a weight-decay coefficient $\lambda>0$ and a learning-rate schedule $\{\eta_t\}_{t\ge0}$ with $0<\eta_t\le 1/(a+\lambda)$ and $\sum_{t\ge0}\eta_t=\infty$. Then gradient descent with constant decoupled weight decay,
\begin{equation}\label{eq:app-const-wd-update}
    w_{t+1}=(1-\eta_t\lambda)w_t-\eta_t\,a(w_t-c)~,
\end{equation}
converges, from every initial point $w_0\in\mathbb{R}$, to
\begin{equation}\label{eq:app-const-wd-limit}
    \bar w=\frac{ac}{a+\lambda}\neq c~,
\end{equation}
and the loss gap at the limit is
\begin{equation}\label{eq:app-const-wd-gap}
    f(\bar w)-f(c)=\frac{a c^2}{2}\left(\frac{\lambda}{a+\lambda}\right)^2>0~.
\end{equation}
In particular the iterates never converge to the minimizer $w^\star=c$.
\end{proposition}

\begin{proof}
Since $\nabla f(w)=a(w-c)$, the update \eqref{eq:app-const-wd-update} is affine in $w_t$. Collecting terms,
\begin{equation}\label{eq:app-const-wd-affine}
    w_{t+1}=\bigl(1-\eta_t(a+\lambda)\bigr)w_t+\eta_t\,ac~.
\end{equation}
A point $\bar w$ is fixed by \eqref{eq:app-const-wd-affine} for every $\eta_t>0$ precisely when $\bigl(1-\eta_t(a+\lambda)\bigr)\bar w+\eta_t ac=\bar w$, that is when
\begin{equation*}
    -\eta_t(a+\lambda)\bar w+\eta_t ac=0
    \iff
    (a+\lambda)\bar w=ac~.
\end{equation*}
Because $a+\lambda>0$, this has the unique solution $\bar w=ac/(a+\lambda)$, which is \eqref{eq:app-const-wd-limit}. Since $\lambda>0$ and $c\neq0$,
\begin{equation}\label{eq:app-const-wd-displacement}
    \bar w-c=\frac{ac}{a+\lambda}-c=-\frac{\lambda c}{a+\lambda}\neq0~,
\end{equation}
so $\bar w\neq c=w^\star$.

Now let $\delta_t=w_t-\bar w$. Subtracting the fixed-point identity $\bar w=\bigl(1-\eta_t(a+\lambda)\bigr)\bar w+\eta_t ac$ from \eqref{eq:app-const-wd-affine} gives
\begin{equation}\label{eq:app-const-wd-error}
    \delta_{t+1}=\bigl(1-\eta_t(a+\lambda)\bigr)\delta_t~.
\end{equation}
Set $\rho_t=1-\eta_t(a+\lambda)$. The hypothesis $0<\eta_t\le 1/(a+\lambda)$ gives $0\le\rho_t<1$, so unrolling \eqref{eq:app-const-wd-error},
\begin{equation*}
    \delta_t=\Bigl(\prod_{s=0}^{t-1}\rho_s\Bigr)\delta_0,
    \qquad
    0\le\prod_{s=0}^{t-1}\rho_s\le\exp\!\Bigl(-(a+\lambda)\sum_{s=0}^{t-1}\eta_s\Bigr)~,
\end{equation*}
where the upper bound uses $1-x\le e^{-x}$. Since $\sum_s\eta_s=\infty$, the exponent tends to $-\infty$, so the product tends to $0$ and therefore $\delta_t\to0$, i.e.\ $w_t\to\bar w$, from every initial point $w_0$.

Finally, using \eqref{eq:app-const-wd-displacement} and $f(c)=0$,
\begin{equation*}
    f(\bar w)-f(c)=\frac{a}{2}(\bar w-c)^2
    =\frac{a c^2}{2}\left(\frac{\lambda}{a+\lambda}\right)^2~,
\end{equation*}
which is \eqref{eq:app-const-wd-gap}.
\end{proof}

\subsection{Non-convergence of constant weight decay for Muon}
\label{app:const-wd-muon}

We restate and prove Proposition~\ref{prop:const-wd-muon}. Consider the shifted matrix quadratic
\begin{equation}\label{eq:app-matrix-quadratic}
    f(W) = \frac{1}{2}\|W - A\|_F^2 = \frac{1}{2}\operatorname{tr}\!\bigl((W-A)^\top(W-A)\bigr),
    \qquad
    \nabla f(W) = W - A~,
\end{equation}
with $A \in \mathbb{R}^{m\times n}$ and unique minimizer $W^\star = A$. Muon with heavy-ball momentum $\beta \in [0,1)$ forms
\begin{align}
    M_t &= \beta M_{t-1} + \nabla f(W_t), \label{eq:app-muon-momentum}\\
    O_t &= \operatorname{polar}(M_t), \label{eq:app-muon-direction}
\end{align}
where $\operatorname{polar}(C) = UV^\top$ for a compact singular value decomposition $C = U\Sigma V^\top$ (and $\operatorname{polar}(0)=0$), so that $\|O_t\|_{\mathrm{op}} \le 1$ for every $t$. The constant-decay Muon update is
\begin{equation}\label{eq:app-muon-const-wd}
    W_{t+1} = (1 - \eta_t\lambda) W_t - \eta_t O_t~.
\end{equation}

\begin{proposition}[restatement of Proposition~\ref{prop:const-wd-muon}]
\label{prop:app-const-wd-muon}
Let $\lambda > 0$, suppose $0 < \eta_t\lambda \le 1$ and $\sum_t \eta_t = \infty$, and consider the Muon iteration \eqref{eq:app-muon-momentum}--\eqref{eq:app-muon-const-wd} on the objective \eqref{eq:app-matrix-quadratic}. Then, regardless of the momentum coefficient $\beta$ and the initialization $W_0$,
\begin{equation}\label{eq:app-muon-radius}
    \limsup_{t\to\infty} \|W_t\|_{\mathrm{op}} \le \frac{1}{\lambda}~.
\end{equation}
Consequently, if $\sigma_{\max}(A) > 1/\lambda$, then
\begin{equation}\label{eq:app-muon-separation}
    \liminf_{t\to\infty} \|W_t - A\|_{\mathrm{op}} \ge \sigma_{\max}(A) - \frac{1}{\lambda} > 0~,
\end{equation}
so the iterates never converge to the minimizer $W^\star = A$, and the loss stays bounded below by
\begin{equation}\label{eq:app-muon-loss-gap}
    \liminf_{t\to\infty} f(W_t) \ge \frac{1}{2}\sum_{i}\Bigl(\sigma_i(A) - \frac{1}{\lambda}\Bigr)_+^2~,
\end{equation}
where $\sigma_i(A)$ are the singular values of $A$ and $(x)_+ = \max(x,0)$.
In particular, for every $\lambda > 0$ there is a shifted quadratic on which constant weight decay fails to recover the minimizer.
\end{proposition}

\begin{proof}
Let $z_t = \|W_t\|_{\mathrm{op}}$. Taking operator norms in \eqref{eq:app-muon-const-wd} and using $\|O_t\|_{\mathrm{op}} \le 1$ together with $1 - \eta_t\lambda \ge 0$,
\begin{equation*}
    z_{t+1} \le (1 - \eta_t\lambda) z_t + \eta_t~.
\end{equation*}
Subtracting $1/\lambda$ from both sides gives
\begin{equation*}
    z_{t+1} - \frac{1}{\lambda} \le (1 - \eta_t\lambda)\Bigl(z_t - \frac{1}{\lambda}\Bigr)~,
\end{equation*}
so by induction
\begin{equation*}
    z_t - \frac{1}{\lambda} \le \Bigl(z_0 - \frac{1}{\lambda}\Bigr)_+ \prod_{s=0}^{t-1}(1 - \eta_s\lambda)~,
\end{equation*}
where $(x)_+ = \max(x,0)$. Since $1 - \eta_s\lambda \le e^{-\eta_s\lambda}$, the product is at most $\exp\bigl(-\lambda\sum_{s<t}\eta_s\bigr)$, which tends to $0$ because $\sum_s \eta_s = \infty$. Hence $\limsup_t z_t \le 1/\lambda$, which is \eqref{eq:app-muon-radius}. The bound uses only $\|O_t\|_{\mathrm{op}} \le 1$, so it is independent of $\beta$ and of how the momentum buffer $M_t$ was formed.

For \eqref{eq:app-muon-separation}, the reverse triangle inequality in operator norm gives
\begin{equation*}
    \|W_t - A\|_{\mathrm{op}} \ge \|A\|_{\mathrm{op}} - \|W_t\|_{\mathrm{op}} = \sigma_{\max}(A) - z_t~,
\end{equation*}
and taking $\liminf$ as $t \to \infty$ together with \eqref{eq:app-muon-radius} gives \eqref{eq:app-muon-separation}, which is strictly positive when $\sigma_{\max}(A) > 1/\lambda$.

For the loss bound \eqref{eq:app-muon-loss-gap}, we use that the Frobenius-norm projection onto the spectral-norm ball clips singular values. Concretely, for any radius $\rho \ge 0$,
\begin{equation}\label{eq:app-clip-projection}
    \min_{\|W\|_{\mathrm{op}} \le \rho} \|W - A\|_F^2 = \sum_i \bigl(\sigma_i(A) - \rho\bigr)_+^2~,
\end{equation}
the minimizer being $\sum_i \min(\sigma_i(A), \rho)\, u_i v_i^\top$ for the singular value decomposition $A = \sum_i \sigma_i(A)\, u_i v_i^\top$. Fix $\rho > 1/\lambda$. By \eqref{eq:app-muon-radius} there is $T_\rho$ such that $\|W_t\|_{\mathrm{op}} \le \rho$ for all $t \ge T_\rho$, so \eqref{eq:app-clip-projection} gives $\|W_t - A\|_F^2 \ge \sum_i (\sigma_i(A) - \rho)_+^2$ for such $t$, and hence
\begin{equation*}
    \liminf_{t\to\infty} f(W_t) = \liminf_{t\to\infty} \tfrac{1}{2}\|W_t - A\|_F^2 \ge \tfrac{1}{2}\sum_i (\sigma_i(A) - \rho)_+^2~.
\end{equation*}
Letting $\rho \downarrow 1/\lambda$ and using continuity of $\rho \mapsto (\sigma_i(A) - \rho)_+^2$ yields \eqref{eq:app-muon-loss-gap}.
\end{proof}

\subsection{Convergence of scaled weight decay for SGD}
\label{app:scaled-wd-sgd}

We restate and prove Theorem~\ref{thm:scaled-wd-sgd-convergence}. Let $F : \mathbb{R}^d \to \mathbb{R}$ be bounded below by $F_{\inf}$ and $L$-smooth, and let the stochastic gradients satisfy
\begin{equation}\label{eq:app-sgd-assumptions}
    \mathbb{E}[g_t \mid \mathcal{F}_t] = \nabla F(x_t), \qquad
    \mathbb{E}[\|g_t\|_2^2 \mid \mathcal{F}_t] \le \sigma^2 + \rho \|\nabla F(x_t)\|_2^2~,
\end{equation}
with $\sigma^2 \ge 0$, $\rho \ge 1$. The scaled weight decay update is
\begin{equation}\label{eq:app-scaled-wd-sgd}
    x_{t+1} = x_t - \eta_t g_t - q\, \eta_t^2 x_t, \qquad q = \frac{\lambda}{\etamax}~.
\end{equation}
We assume the iterates are bounded in mean square, $\sup_t \mathbb{E}\|x_t\|_2^2 \le B^2 < \infty$, which in the nonconvex setting plays the role that strong convexity plays for the distance recursion: it keeps the order-$\eta_t^2$ shrink term summable.

\begin{theorem}[restatement of Theorem~\ref{thm:scaled-wd-sgd-convergence}]
\label{thm:app-scaled-wd-sgd}
Under \eqref{eq:app-sgd-assumptions}, $\sup_t \mathbb{E}\|x_t\|_2^2 \le B^2 < \infty$, and the Robbins--Monro conditions $\sum_t \eta_t = \infty$, $\sum_t \eta_t^2 < \infty$, the iterates \eqref{eq:app-scaled-wd-sgd} satisfy
\begin{equation}\label{eq:app-scaled-wd-stationarity}
    \frac{1}{S_T}\sum_{t=0}^{T-1} \eta_t\, \mathbb{E}\|\nabla F(x_t)\|_2^2 \to 0,
    \qquad S_T = \sum_{t=0}^{T-1}\eta_t~,
\end{equation}
and in particular $\liminf_{t\to\infty} \mathbb{E}\|\nabla F(x_t)\|_2^2 = 0$.
\end{theorem}

\begin{proof}
By $L$-smoothness, with the increment $\Delta_t = x_{t+1} - x_t = -\eta_t g_t - q\eta_t^2 x_t$,
\begin{equation}\label{eq:app-descent-lemma}
    F(x_{t+1}) \le F(x_t) + \langle \nabla F(x_t), \Delta_t\rangle + \frac{L}{2}\|\Delta_t\|_2^2~.
\end{equation}
Take the conditional expectation given $\mathcal{F}_t$. Using $\mathbb{E}[g_t \mid \mathcal{F}_t] = \nabla F(x_t)$,
\begin{equation}\label{eq:app-inner-product}
    \mathbb{E}[\langle \nabla F(x_t), \Delta_t\rangle \mid \mathcal{F}_t]
    = -\eta_t \|\nabla F(x_t)\|_2^2 - q\eta_t^2 \langle \nabla F(x_t), x_t\rangle~.
\end{equation}
The shrink cross term is controlled by Cauchy--Schwarz and Young's inequality,
\begin{equation}\label{eq:app-cross-young}
    -q\eta_t^2 \langle \nabla F(x_t), x_t\rangle
    \le q\eta_t^2 \|\nabla F(x_t)\|_2 \|x_t\|_2
    \le \frac{\eta_t}{4}\|\nabla F(x_t)\|_2^2 + q^2\eta_t^3\|x_t\|_2^2~,
\end{equation}
where the last step uses $ab \le \tfrac14 a^2 + b^2$ with $a = \sqrt{\eta_t}\|\nabla F(x_t)\|_2$ and $b = q\eta_t^{3/2}\|x_t\|_2$. For the second-order term, $\|\Delta_t\|_2^2 \le 2\eta_t^2\|g_t\|_2^2 + 2q^2\eta_t^4\|x_t\|_2^2$, so by \eqref{eq:app-sgd-assumptions},
\begin{equation}\label{eq:app-second-order}
    \frac{L}{2}\mathbb{E}[\|\Delta_t\|_2^2 \mid \mathcal{F}_t]
    \le L\eta_t^2\bigl(\sigma^2 + \rho\|\nabla F(x_t)\|_2^2\bigr) + Lq^2\eta_t^4\|x_t\|_2^2~.
\end{equation}
Substituting \eqref{eq:app-inner-product}--\eqref{eq:app-second-order} into the conditional expectation of \eqref{eq:app-descent-lemma},
\begin{equation*}
    \mathbb{E}[F(x_{t+1}) \mid \mathcal{F}_t]
    \le F(x_t) - \eta_t\Bigl(\tfrac34 - L \rho \eta_t\Bigr)\|\nabla F(x_t)\|_2^2
    + L\sigma^2\eta_t^2 + q^2\eta_t^3\|x_t\|_2^2 + Lq^2\eta_t^4\|x_t\|_2^2~.
\end{equation*}
Since $\eta_t \to 0$, there is an index beyond which $\tfrac34 - L \rho\eta_t \ge \tfrac12$. Taking total expectations, using $\mathbb{E}\|x_t\|_2^2 \le B^2$, and summing from that index to $T-1$,
\begin{equation*}
    \frac12 \sum_{t}\eta_t \mathbb{E}\|\nabla F(x_t)\|_2^2
    \le \mathbb{E} F(x_{t_1}) - F_{\inf}
    + \bigl(L\sigma^2 + q^2 B^2 + Lq^2 B^2\bigr)\sum_{t}\eta_t^2~,
\end{equation*}
where we bounded $\eta_t^3, \eta_t^4 \le \eta_t^2$ for large $t$. The right-hand side is finite because $\sum_t \eta_t^2 < \infty$. Dividing by $S_T$ and letting $T \to \infty$, and noting $S_T \to \infty$, gives \eqref{eq:app-scaled-wd-stationarity}. The $\liminf$ statement follows since a nonnegative sequence whose $\eta_t$-weighted average tends to zero, with $\sum_t \eta_t = \infty$, has a subsequence tending to zero.
\end{proof}

When the objective is in addition strongly convex, the stationarity guarantee upgrades to convergence of the iterates to the unique minimizer, and the bounded-iterate assumption becomes unnecessary. Both this statement and its rate follow from a single stochastic-approximation recursion. 

\begin{lemma}\label{lem:app-sa-recursion}
Let $\{u_t\}$ be nonnegative reals satisfying, for all sufficiently large $t$,
\begin{equation}\label{eq:app-sa-recursion}
    u_{t+1} \le (1 - \alpha\eta_t) u_t + \kappa \eta_t^2
\end{equation}
for constants $\alpha, \kappa > 0$ with $\alpha\eta_t \le 1$. If $\sum_t \eta_t = \infty$ and $\sum_t \eta_t^2 < \infty$, then $u_t \to 0$. For $\eta_t = \gamma/(t+t_0)$ with $\alpha\gamma > 1$, one has $u_t = O(1/t)$.
\end{lemma}

\begin{proof}
Unrolling \eqref{eq:app-sa-recursion} from a base index $t_1$ beyond which it holds,
\begin{equation*}
    u_{t+1} \le \Bigl(\prod_{s=t_1}^{t}(1-\alpha\eta_s)\Bigr) u_{t_1}
    + \kappa \sum_{s=t_1}^{t} \eta_s^2 \prod_{r=s+1}^{t}(1-\alpha\eta_r)~.
\end{equation*}
Using $1 - \alpha\eta_s \le e^{-\alpha\eta_s}$, the leading product is at most $\exp(-\alpha\sum_{s=t_1}^t \eta_s) \to 0$ since $\sum_s \eta_s = \infty$. For the second term, fix $\epsilon > 0$; because $\sum_s \eta_s^2 < \infty$, choose $t_2 \ge t_1$ with $\sum_{s \ge t_2} \eta_s^2 < \epsilon$. The tail $s \ge t_2$ contributes at most $\kappa\sum_{s\ge t_2}\eta_s^2 < \kappa\epsilon$ (each product factor is at most $1$), while the fixed prefix $t_1 \le s < t_2$ is multiplied by $\exp(-\alpha\sum_{r=t_2}^t \eta_r) \to 0$. Hence $\limsup_t u_{t+1} \le \kappa\epsilon$ for every $\epsilon$, so $u_t \to 0$. For the harmonic schedule, a standard induction on \eqref{eq:app-sa-recursion} with $\eta_t = \gamma/(t+t_0)$ and $\alpha\gamma > 1$ gives $u_t \le \nu/(t+t_0)$ with $\nu = \max\{(t_1+t_0)u_{t_1}, \kappa\gamma^2/(\alpha\gamma-1)\}$, which is $O(1/t)$.
\end{proof}

\begin{theorem}[restatement of Theorem~\ref{thm:scaled-wd-sgd-strongly-convex}]
\label{thm:app-scaled-wd-sgd-sc}
Under \eqref{eq:app-sgd-assumptions}, suppose $F$ is in addition $\mu$-strongly convex with $\mu \le L$ and unique minimizer $x^\star$, and the Robbins--Monro conditions hold. Then the iterates \eqref{eq:app-scaled-wd-sgd} satisfy $\mathbb{E}\|x_t - x^\star\|_2^2 \to 0$, with rate $O(1/t)$ for a harmonic schedule with $\mu\gamma > 1$.
\end{theorem}

\begin{proof}
Write $e_t = x_t - x^\star$. Expanding \eqref{eq:app-scaled-wd-sgd} and conditioning on $\mathcal{F}_t$, using $\mathbb{E}[g_t \mid \mathcal{F}_t] = \nabla F(x_t)$,
\begin{equation}\label{eq:app-sc-expand}
    \mathbb{E}[\|e_{t+1}\|_2^2 \mid \mathcal{F}_t]
    = \|e_t\|_2^2
    - 2\eta_t \langle e_t, \nabla F(x_t)\rangle
    - 2q\eta_t^2 \langle e_t, x_t\rangle
    + \eta_t^2 \mathbb{E}[\|g_t + q\eta_t x_t\|_2^2 \mid \mathcal{F}_t]~.
\end{equation}
Strong convexity and $\nabla F(x^\star) = 0$ give $\langle e_t, \nabla F(x_t)\rangle \ge \mu\|e_t\|_2^2$, and $L$-smoothness gives $\|\nabla F(x_t)\|_2 \le L\|e_t\|_2$. Writing $x_t = e_t + x^\star$, Young's inequality gives
\begin{equation}\label{eq:app-sc-young}
    -2\langle e_t, x_t\rangle
    = -2\|e_t\|_2^2 - 2\langle e_t, x^\star\rangle
    \le -\|e_t\|_2^2 + \|x^\star\|_2^2~,
\end{equation}
so strong convexity here also supplies the mean-square boundedness of the iterates, and no separate bounded-iterate assumption is required. For the second-moment term, using $\|g_t + q\eta_t x_t\|_2^2 \le 2\|g_t\|_2^2 + 2q^2\eta_t^2\|x_t\|_2^2$, \eqref{eq:app-sgd-assumptions}, $\|\nabla F(x_t)\|_2^2 \le L^2\|e_t\|_2^2$, and $\|x_t\|_2^2 \le 2\|e_t\|_2^2 + 2\|x^\star\|_2^2$,
\begin{equation}\label{eq:app-sc-second-moment}
    \mathbb{E}[\|g_t + q\eta_t x_t\|_2^2 \mid \mathcal{F}_t]
    \le 2\sigma^2 + 2\rho L^2 \|e_t\|_2^2 + 4q^2\eta_t^2\bigl(\|e_t\|_2^2 + \|x^\star\|_2^2\bigr)~.
\end{equation}
Substituting into \eqref{eq:app-sc-expand} and collecting powers of $\eta_t$,
\begin{equation*}
    \mathbb{E}[\|e_{t+1}\|_2^2 \mid \mathcal{F}_t]
    \le \bigl(1 - 2\mu\eta_t + K_1\eta_t^2 + K_2\eta_t^4\bigr)\|e_t\|_2^2
    + \bigl(q\|x^\star\|_2^2 + 2\sigma^2\bigr)\eta_t^2 + 4q^2\|x^\star\|_2^2\,\eta_t^4~,
\end{equation*}
with $K_1 = 2\rho L^2 + 4q^2$ and $K_2 = 4q^2$ (the $-2q\eta_t^2\|e_t\|_2^2$ term only helps). Since $\eta_t \to 0$, for all sufficiently large $t$ the bracketed coefficient is at most $1 - \mu\eta_t$ and the $\eta_t^4$ terms are dominated by $\eta_t^2$. Taking total expectations yields, for large $t$,
\begin{equation*}
    \mathbb{E}\|e_{t+1}\|_2^2 \le (1 - \mu\eta_t)\,\mathbb{E}\|e_t\|_2^2 + \kappa\eta_t^2
\end{equation*}
for some $\kappa > 0$. This is the recursion \eqref{eq:app-sa-recursion} with $u_t = \mathbb{E}\|e_t\|_2^2$ and $\alpha = \mu$, so Lemma~\ref{lem:app-sa-recursion} gives $\mathbb{E}\|e_t\|_2^2 \to 0$, and the $O(1/t)$ rate for a harmonic schedule with $\mu\gamma > 1$.
\end{proof}

\subsection{Convergence of scaled weight decay for Muon}
\label{app:scaled-wd-muon}

We restate and prove the Muon results of the main text. Let $f : \mathbb{R}^{m\times n} \to \mathbb{R}$ be bounded below by $f_{\inf}$ and $L$-smooth in Frobenius norm, with stochastic gradient $G_t = \nabla f(W_t) + \xi_t$ where $\mathbb{E}[\xi_t \mid \mathcal{F}_t] = 0$ and $\mathbb{E}[\|\xi_t\|_F^2 \mid \mathcal{F}_t] \le \sigma^2/B$ for a fixed batch size $B$. The buffer is $M_t = \beta M_{t-1} + G_t$ with fixed $\beta \in [0,1)$, $O_t = \mathrm{polar}(M_t)$, and the scaled-decay update is
\begin{equation}\label{eq:app-muon-scaled-wd}
    W_{t+1} = W_t - \eta_t O_t - q\, \eta_t^2 W_t, \qquad q = \frac{\lambda}{\etamax}~.
\end{equation}
Write $r = \min(m,n)$ for the ambient rank bound, and assume the iterates are bounded in mean square, $\sup_t \mathbb{E}\|W_t\|_F^2 \le \mathcal{B}^2$.

Unrolling from $M_{-1} = 0$ gives $M_t = \sum_{s\le t}\beta^{t-s}G_s$, so the normalized buffer $\bar M_t = (1-\beta) M_t$ obeys $\bar M_t = \beta \bar M_{t-1} + (1-\beta)G_t$ and is an exact positive multiple of $M_t$. Since orthogonalization is scale-invariant, $O_t = \mathrm{polar}(M_t) = \mathrm{polar}(\bar M_t)$, and we track the error of the normalized buffer,
\begin{equation}\label{eq:app-muon-error}
    E_t = \bar M_t - \nabla f(W_t)~.
\end{equation}
The orthogonalization satisfies the polar identities
\begin{equation}\label{eq:app-polar-identities}
    \langle \bar M_t, O_t\rangle_F = \|\bar M_t\|_*, \qquad \|O_t\|_{\mathrm{op}} \le 1, \qquad \|O_t\|_F^2 = \mathrm{rank}(M_t) \le r~,
\end{equation}
where $\|\cdot\|_*$ is the nuclear norm, dual to the operator norm. We use repeatedly $\langle X, Y\rangle_F \le \|X\|_* \|Y\|_{\mathrm{op}}$ and $\|X\|_* \le \sqrt{\mathrm{rank}(X)}\,\|X\|_F$.

\begin{lemma}[alignment]\label{lem:app-muon-alignment}
For every $t$,
\begin{equation}\label{eq:app-muon-alignment}
    \langle \nabla f(W_t), O_t\rangle_F \ge \|\nabla f(W_t)\|_* - 2\sqrt{r}\,\|E_t\|_F~.
\end{equation}
\end{lemma}

\begin{proof}
Writing $\nabla f(W_t) = \bar M_t - E_t$ and using \eqref{eq:app-polar-identities},
\begin{equation*}
    \langle \nabla f(W_t), O_t\rangle_F
    = \|\bar M_t\|_* - \langle E_t, O_t\rangle_F
    \ge \|\bar M_t\|_* - \|E_t\|_*~,
\end{equation*}
since $\langle E_t, O_t\rangle_F \le \|E_t\|_*\|O_t\|_{\mathrm{op}} \le \|E_t\|_*$. By the reverse triangle inequality, $\|\bar M_t\|_* \ge \|\nabla f(W_t)\|_* - \|E_t\|_*$; combining and using $\|E_t\|_* \le \sqrt{r}\,\|E_t\|_F$ gives \eqref{eq:app-muon-alignment}.
\end{proof}

The tracking error obeys a contraction with a bias term of order $\eta$ and a noise term that does not vanish for fixed $\beta$ and $B$.

\begin{lemma}[tracking with a noise floor]\label{lem:app-muon-tracking}
Under the standing assumptions, if $\eta_t \to 0$, then
\begin{equation}\label{eq:app-muon-tracking-floor}
    \limsup_{t\to\infty} \mathbb{E}\|E_t\|_F^2 \le (1-\beta)\,\frac{\sigma^2}{B}~.
\end{equation}
\end{lemma}

\begin{proof}
If $\beta=0$, then $\bar M_t=G_t$ and $E_t=\xi_t$, so the claimed bound follows immediately from $\mathbb{E}\|\xi_t\|_F^2\le \sigma^2/B$. Suppose henceforth that $0<\beta<1$. From $\bar M_t = \beta \bar M_{t-1} + (1-\beta)G_t$ and $G_t = \nabla f(W_t) + \xi_t$,
\begin{equation}\label{eq:app-muon-error-recursion}
    E_t = \beta E_{t-1} + \beta\,\delta_t + (1-\beta)\xi_t,
    \qquad \delta_t = \nabla f(W_{t-1}) - \nabla f(W_t)~.
\end{equation}
The term $\beta E_{t-1} + \beta\delta_t$ is $\mathcal{F}_t$-measurable and $\mathbb{E}[\xi_t\mid\mathcal{F}_t] = 0$, so the cross term vanishes and
\begin{equation*}
    \mathbb{E}\|E_t\|_F^2 = \beta^2\,\mathbb{E}\|E_{t-1} + \delta_t\|_F^2 + (1-\beta)^2\,\mathbb{E}\|\xi_t\|_F^2~.
\end{equation*}
By Young's inequality with weight $a = (1-\beta)/\beta$, $\|E_{t-1}+\delta_t\|_F^2 \le \tfrac{1}{\beta}\|E_{t-1}\|_F^2 + \tfrac{1}{1-\beta}\|\delta_t\|_F^2$, so $\beta^2\|E_{t-1}+\delta_t\|_F^2 \le \beta\|E_{t-1}\|_F^2 + \tfrac{\beta^2}{1-\beta}\|\delta_t\|_F^2$. Using $\mathbb{E}\|\xi_t\|_F^2 \le \sigma^2/B$,
\begin{equation}\label{eq:app-muon-tracking-recursion}
    \mathbb{E}\|E_t\|_F^2 \le \beta\,\mathbb{E}\|E_{t-1}\|_F^2 + \frac{\beta^2}{1-\beta}\,\mathbb{E}\|\delta_t\|_F^2 + (1-\beta)^2\frac{\sigma^2}{B}~.
\end{equation}
By $L$-smoothness and \eqref{eq:app-muon-scaled-wd}, $\|\delta_t\|_F \le L\|W_t - W_{t-1}\|_F \le L(\eta_{t-1}\sqrt{r} + q\eta_{t-1}^2\|W_{t-1}\|_F)$, so $\mathbb{E}\|\delta_t\|_F^2 \le 2L^2(r\eta_{t-1}^2 + q^2\eta_{t-1}^4\mathcal{B}^2) \to 0$. Thus \eqref{eq:app-muon-tracking-recursion} is $u_t \le \beta u_{t-1} + c_t$ with $u_t = \mathbb{E}\|E_t\|_F^2$ and $\limsup_t c_t = (1-\beta)^2\sigma^2/B$. Taking $\limsup$ gives $\limsup_t u_t \le \beta\limsup_t u_t + (1-\beta)^2\sigma^2/B$, hence $\limsup_t u_t \le (1-\beta)\sigma^2/B$.
\end{proof}

\begin{theorem}[restatement of Theorem~\ref{thm:scaled-wd-muon-stationarity}]
\label{thm:app-scaled-wd-muon}
Under the standing assumptions, if $\sum_t \eta_t = \infty$ and $\sum_t \eta_t^2 < \infty$, then
\begin{equation}\label{eq:app-muon-scaled-stationarity}
    \limsup_{T\to\infty}\ \frac{1}{S_T}\sum_{t=0}^{T-1} \eta_t\, \mathbb{E}\|\nabla f(W_t)\|_*
    \le 2\sigma\sqrt{\frac{(1-\beta)\,r}{B}},
    \qquad S_T = \sum_{t=0}^{T-1}\eta_t~.
\end{equation}
\end{theorem}

\begin{proof}
By $L$-smoothness, with $\Delta_t = -\eta_t O_t - q\eta_t^2 W_t$,
\begin{equation}\label{eq:app-muon-descent}
    f(W_{t+1}) \le f(W_t) + \langle \nabla f(W_t), \Delta_t\rangle_F + \frac{L}{2}\|\Delta_t\|_F^2~.
\end{equation}
The first-order term splits into $-\eta_t\langle\nabla f(W_t), O_t\rangle_F - q\eta_t^2\langle\nabla f(W_t), W_t\rangle_F$. By Lemma~\ref{lem:app-muon-alignment}, the first piece is at most $-\eta_t\|\nabla f(W_t)\|_* + 2\sqrt{r}\,\eta_t\|E_t\|_F$. For the shrink term, Frobenius Cauchy--Schwarz and $L$-smoothness give
\begin{align*}
    \mathbb{E}\bigl|\langle\nabla f(W_t),W_t\rangle_F\bigr|
    &\le \mathbb{E}\bigl[\|\nabla f(W_t)\|_F\|W_t\|_F\bigr]\\
    &\le \|\nabla f(0)\|_F\,\mathcal{B}+L\mathcal{B}^2
    =: C_W~.
\end{align*}
For the second-order term, $\|\Delta_t\|_F^2 \le 2r\eta_t^2 + 2q^2\eta_t^4\|W_t\|_F^2$. Taking total expectations in \eqref{eq:app-muon-descent} and substituting these bounds yields
\begin{equation}\label{eq:app-muon-rearranged}
    \eta_t\,\mathbb{E}\|\nabla f(W_t)\|_*
    \le \mathbb{E}[f(W_t)-f(W_{t+1})] + 2\sqrt{r}\,\eta_t\mathbb{E}\|E_t\|_F
    + (qC_W+Lr)\eta_t^2 + Lq^2\mathcal{B}^2\eta_t^4~.
\end{equation}
Choose $t_1$ so that $\eta_t\le1$ for every $t\ge t_1$. Summing from $t_1$ to $T-1$,
\begin{equation*}
    \sum_{t=t_1}^{T-1} \eta_t\,\mathbb{E}\|\nabla f(W_t)\|_*
    \le \mathbb{E} f(W_{t_1}) - f_{\inf}
    + 2\sqrt{r}\sum_{t=t_1}^{T-1} \eta_t\,\mathbb{E}\|E_t\|_F
    + \bigl(qC_W+Lr+Lq^2\mathcal{B}^2\bigr)\sum_{t=t_1}^{T-1} \eta_t^2~.
\end{equation*}
Here we used $\eta_t^4\le \eta_t^2$ for $t\ge t_1$.
Divide by $S_T$. The telescoped gap and the finite $\sum_t\eta_t^2$ term vanish after division by $S_T \to \infty$. For the tracking term, $\mathbb{E}\|E_t\|_F \le \sqrt{\mathbb{E}\|E_t\|_F^2}$, and by Lemma~\ref{lem:app-muon-tracking} $\limsup_t \mathbb{E}\|E_t\|_F \le \sqrt{(1-\beta)\sigma^2/B}$; a weighted Cesàro average has $\limsup$ at most this bound. Hence
\begin{equation*}
    \limsup_T \frac{1}{S_T}\sum_t \eta_t\,\mathbb{E}\|\nabla f(W_t)\|_*
    \le 2\sqrt{r}\,\sqrt{\frac{(1-\beta)\sigma^2}{B}} = 2\sigma\sqrt{\frac{(1-\beta)r}{B}}~,
\end{equation*}
which is \eqref{eq:app-muon-scaled-stationarity}. The scaled shrink term contributes only the order-$\eta_t^2$ and order-$\eta_t^4$ terms in \eqref{eq:app-muon-rearranged}, both summable, so it does not affect the floor; the same bound holds for Muon without weight decay.
\end{proof}

The floor vanishes in the deterministic case $\sigma = 0$, in which case the right-hand side of \eqref{eq:app-muon-scaled-stationarity} is zero and $\tfrac{1}{S_T}\sum_t \eta_t\,\mathbb{E}\|\nabla f(W_t)\|_* \to 0$.

\begin{theorem}[restatement of Theorem~\ref{thm:scaled-wd-muon-convex}]
\label{thm:app-scaled-wd-muon-convex}
Suppose the assumptions of Theorem~\ref{thm:app-scaled-wd-muon} hold with $\sigma = 0$, and additionally suppose that $\|W_t\|_F\le \mathcal{B}$ almost surely for every $t$. If $f$ is convex with minimizer $W^\star$ and $f^\star = f(W^\star)$, then
\begin{equation*}
    \frac{1}{S_T}\sum_{t=0}^{T-1}\eta_t\,\mathbb{E}[f(W_t)-f^\star]\to0,
    \qquad
    \liminf_{t\to\infty}\mathbb{E}[f(W_t)-f^\star]=0~.
\end{equation*}
\end{theorem}

\begin{proof}
By convexity and duality,
\begin{equation*}
    f(W_t) - f^\star \le \langle \nabla f(W_t), W_t - W^\star\rangle_F
    \le \|\nabla f(W_t)\|_*\bigl(\|W_t\|_F + \|W^\star\|_F\bigr)~.
\end{equation*}
Using the almost-sure bound $\|W_t\|_F\le\mathcal{B}$ and then taking expectations gives $\mathbb{E}[f(W_t) - f^\star] \le (\mathcal{B} + \|W^\star\|_F)\,\mathbb{E}\|\nabla f(W_t)\|_*$. Multiplying by $\eta_t$, summing, and dividing by $S_T$,
\begin{equation*}
    \frac{1}{S_T}\sum_{t=0}^{T-1}\eta_t\,\mathbb{E}[f(W_t) - f^\star]
    \le (\mathcal{B} + \|W^\star\|_F)\,\frac{1}{S_T}\sum_{t=0}^{T-1}\eta_t\,\mathbb{E}\|\nabla f(W_t)\|_* \to 0~.
\end{equation*}
This proves the weighted-average statement. Since the summands are nonnegative and $\sum_t\eta_t=\infty$, it also implies $\liminf_t\mathbb{E}[f(W_t)-f^\star]=0$.
\end{proof}

\section{Additional Data from MoE Training Experiments}
\label{app:additional-data}

This appendix collects supporting diagnostics for the weight-norm analysis of Section~\ref{sec:steady-state}, together with the loss and norm curves for the intermediate widths. All diagnostics are computed over the matrix-valued, non-router weights updated by Muon.

\begin{figure}[htbp]
    \centering
    \includegraphics[width=0.5\linewidth]{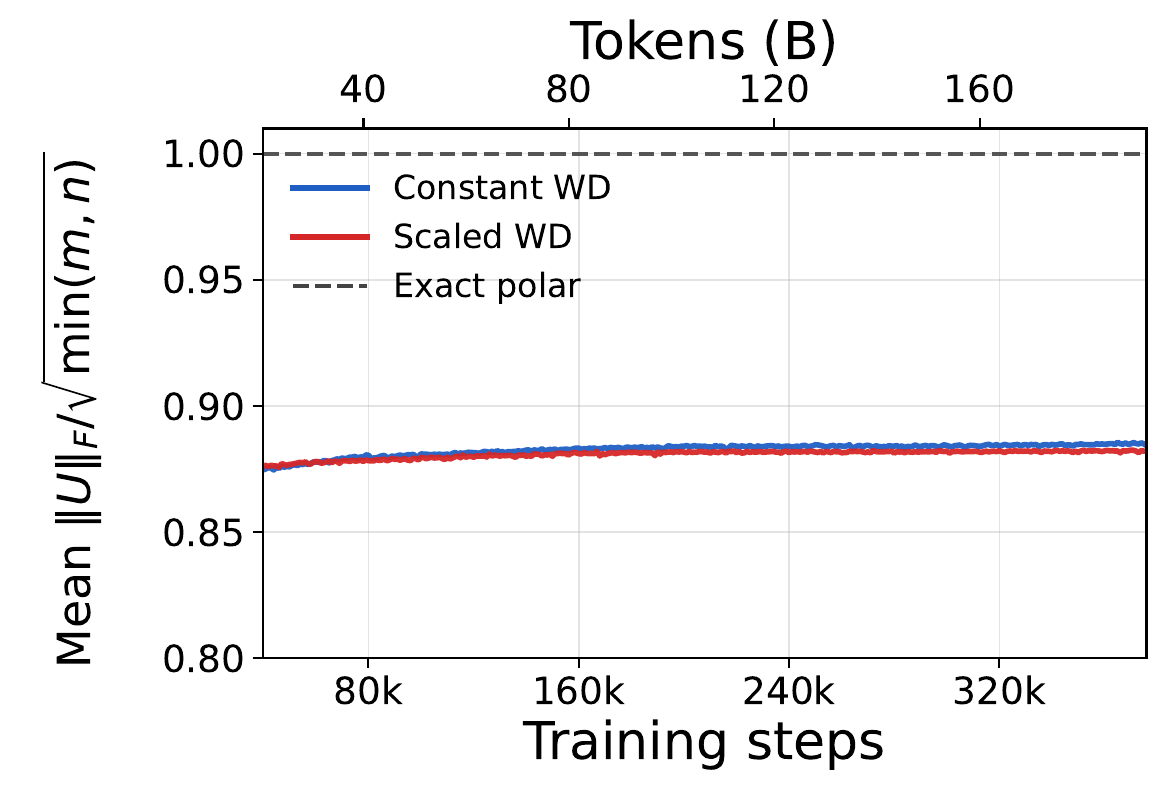}
    \caption{\textbf{Imperfect orthogonalization of the five-step Newton--Schulz polar map (width 1024).} We plot the layer-averaged normalized Frobenius norm $\tfrac{1}{L}\sum_\ell \|U_t^{(\ell)}\|_F / \sqrt{\min(m_\ell, n_\ell)}$ of the applied Muon update. An exact full-rank polar factor would give a ratio of one; the five-step Newton--Schulz output instead stays near $0.88$ throughout training.}
    \label{fig:app-ns5}
\end{figure}

\begin{figure}[htbp]
    \centering
    \includegraphics[width=0.9\linewidth]{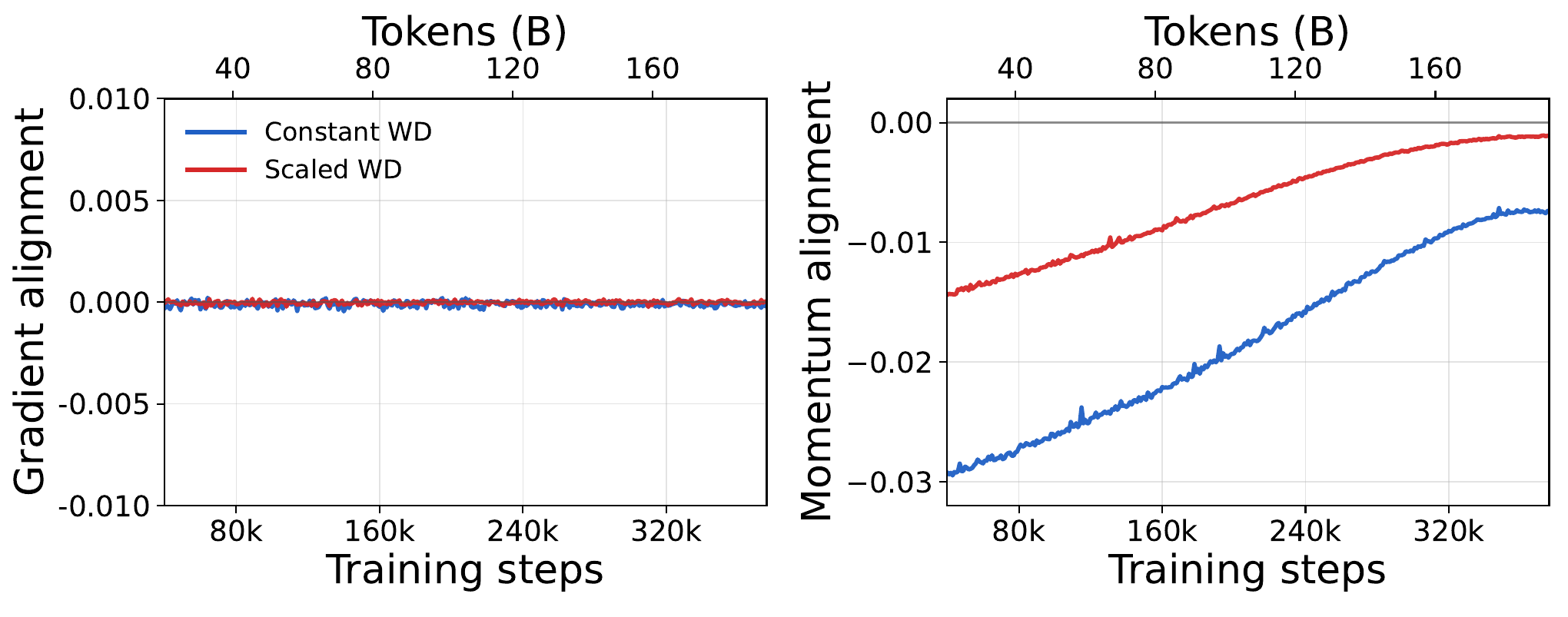}
    \caption{\textbf{Weight--gradient and weight--momentum alignment (width 1024).} \emph{Left:} the raw gradient stays nearly orthogonal to the weights under both decay rules, $\langle W_t, G_t\rangle \approx 0$, as expected from the scale-invariance induced by the normalization layers. \emph{Right:} the momentum buffer which is an exponential average of past gradients, acquires an appreciably negative alignment with the weights. As the learning rate decays the newly injected radial component shrinks relative to the weight norm, so the momentum alignment moves toward zero.}
    \label{fig:app-gm-alignment}
\end{figure}

\begin{figure}[htbp]
    \centering
    \includegraphics[width=0.55\linewidth]{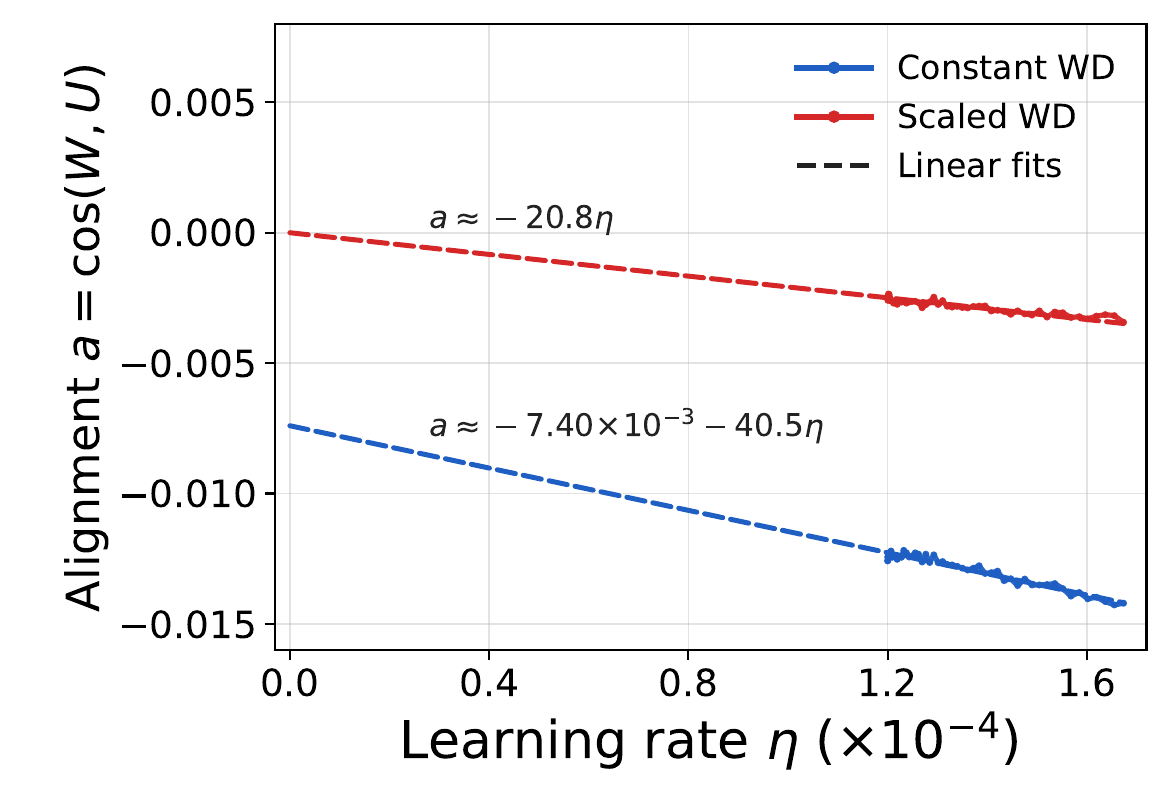}
    \caption{\textbf{Late-run weight--update alignment versus learning rate (width 1024).} Aggregate weight--update alignment $a$ against the learning rate $\eta$ over the final $50$k steps (steps $326$k--$376$k); the observed samples lie on the right, and the dashed fits are extended toward $\eta = 0$ to expose their asymptotics. Scaled weight decay is well described by the origin-constrained relation $a \approx -20.8\,\eta$, so that $a/\eta$ approaches a finite constant. Constant weight decay instead requires a shifted fit $a \approx -7.40\times 10^{-3} - 40.5\,\eta$ with a nonzero intercept. This changes the asymptotics of the steady-state norm and underlies the constant-norm versus collapsing-norm behavior of the two rules. The extrapolated portion comes from modeling the late-run fits, and not directly observed since the minimum learning rate in the experiment is $\etamax/10$.}
    \label{fig:app-late-alignment}
\end{figure}

\begin{figure}[htbp]
    \centering
    \includegraphics[width=0.9\linewidth]{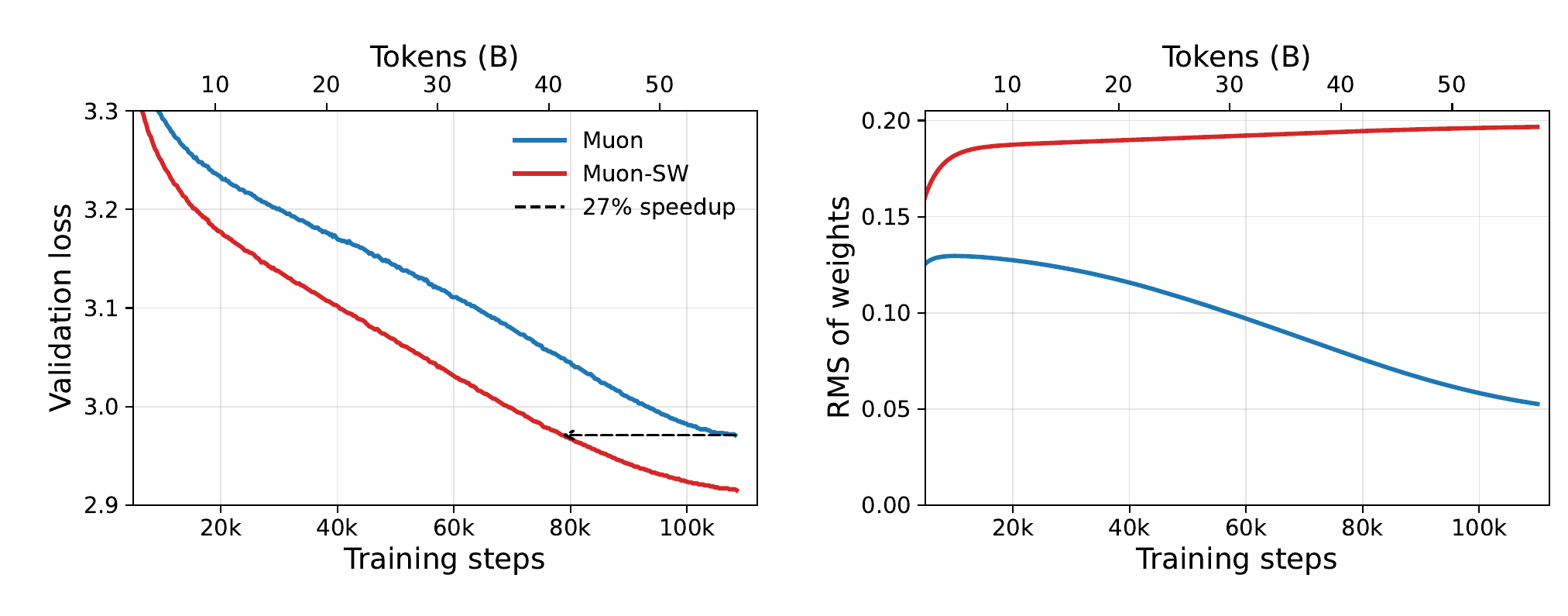}
    \caption{\textbf{Validation loss and weight norm at width 512.} \emph{Left:} validation loss for \textcolor{blue}{Muon} (constant weight decay) and \textcolor{red}{Muon-SW} (scaled weight decay); the dashed arrow marks the reduction in steps for the scaled run to reach the best constant-decay loss of $2.971$ ($108{,}250$ vs.\ $78{,}500$ steps, a ${\sim}27\%$ reduction). \emph{Right:} the scaled-decay weight norm rises quickly and remains near $0.20$, while the constant-decay norm peaks early and then decays to ${\sim}0.05$, reproducing the width-1024 separation.}
    \label{fig:app-w512}
\end{figure}

\begin{figure}[htbp]
    \centering
    \includegraphics[width=0.9\linewidth]{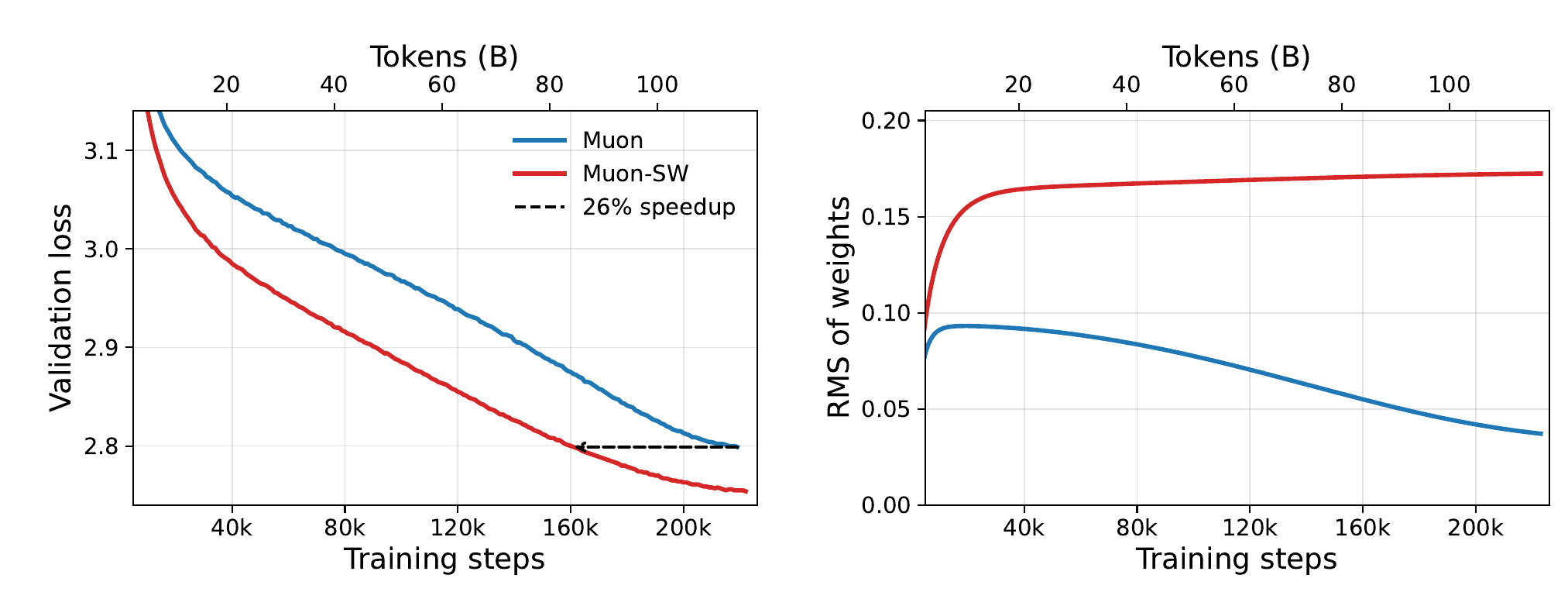}
    \caption{\textbf{Validation loss and weight norm at width 768.} \emph{Left:} the best constant-decay loss of $2.799$ is reached at $219{,}000$ steps, and Muon-SW reaches it at $161{,}000$ steps, a ${\sim}26\%$ reduction. \emph{Right:} scaled weight decay holds the norm near $0.17$ while the constant-decay norm turns over and decays to ${\sim}0.04$. Together with the width-512 result, this confirms that the loss improvement and contrasting norm dynamics are not specific to width 1024.}
    \label{fig:app-w768}
\end{figure}

\FloatBarrier
\section{PyTorch Implementation of MuonSW}
\label{app:implementation}
We present a PyTorch \cite{paszke2019pytorchimperativestylehighperformance} implementation of Muon-SW, the Muon optimizer with scaled weight decay. Our implementation is based on the code by Jordan \cite{jordan2024muon}. The important change from standard decoupled decay is that the decay coefficient uses the scheduler-mutated learning rate $\eta_t$ and the peak learning rate $\eta_{\max}$, so the parameter shrinkage is $c_t=\lambda\eta_t^2/\eta_{\max}$ instead of $c_t=\lambda\eta_t$. The Muon step itself includes the $\alpha\sqrt{\max(m,n)}$ adjustment for an $m\times n$ matrix, with default $\alpha=0.2$ it matches the RMS of AdamW updates and allows one to use the same nominal learning rate as AdamW \cite{liu2025muonscalablellmtraining}. The Muon optimizer should be used for hidden matrix parameters such as attention projection matrices and expert MLP matrices. Embeddings, the tied un-embedding or language-model head, normalization weights, biases, and other vector or scalar parameters should be trained with AdamW. 

The implementation uses five Newton--Schulz iterations with coefficients $(3.4445,-4.7750,2.0315)$, matching the common Muon implementation \cite{jordan2024muon}. One could instead use other polynomial coefficients as in Polar Express \cite{amsel2026polarexpressoptimalmatrix}, or use randomized orthogonalization methods as proposed in \cite{ahn2025dion, ahn2025dion2}. We have provided this particular implementation since it was used in our experiments, but the scaled weight decay method can be implemented by changing just a few lines of code in other optimizers such as Shampoo \cite{gupta2018, anil2021}, SOAP \cite{vyas2025}, Schedule-Free Muon \cite{apte2026anytimetrainingschedulefreespectral} and Aurora \cite{dewulf2026auroraleverageawarespectraloptimizer}.

\clearpage
\newgeometry{margin=0.5in}
\begin{center}
\begin{lstlisting}[language=Python,basicstyle=\ttfamily\footnotesize,breaklines=true]
import math, torch
import torch.distributed as dist

def zeropower_via_newtonschulz5(G, steps=5, eps=1e-7):
    assert G.ndim == 2
    a, b, c = (3.4445, -4.7750, 2.0315)
    X = G.bfloat16()
    transposed = X.size(0) > X.size(1)
    if transposed:
        X = X.T
    X = X / (X.norm() + eps)
    for _ in range(steps):
        A = X @ X.T
        B = b * A + c * A @ A
        X = a * X + B @ X
    return X.T if transposed else X

def muon_direction(grad, momentum, beta=0.95, ns_steps=5, nesterov=True, shape_scale=0.2):
    momentum.lerp_(grad, 1.0 - beta)
    update = grad.lerp(momentum, beta) if nesterov else momentum
    original_shape = update.shape
    if update.ndim == 4:
        update = update.view(update.size(0), -1)
    update = zeropower_via_newtonschulz5(update, steps=ns_steps)
    m, n = update.shape
    update.mul_(shape_scale * math.sqrt(max(m, n)))
    return update.reshape(original_shape)

class DistributedMuonSW(torch.optim.Optimizer):
    def __init__(self, params, lr=1e-2, eta_max=None, lambda_=0.1,
                 shape_scale=0.2, momentum=0.95, nesterov=True, ns_steps=5):
        defaults = dict(lr=lr, eta_max=lr if eta_max is None else eta_max,
                        lambda_=lambda_, shape_scale=shape_scale,
                        momentum=momentum, nesterov=nesterov, ns_steps=ns_steps)
        assert isinstance(params, list) and len(params) >= 1 and isinstance(params[0], torch.nn.Parameter)
        params = sorted(params, key=lambda x: x.size(), reverse=True)
        super().__init__(params, defaults)

    @staticmethod
    def scaled_decay_coeff(group):
        lr, eta_max, lambda_ = group["lr"], group["eta_max"], group["lambda_"]
        # MuonSW: standard WD would return lambda_ * lr.
        return 0.0 if eta_max <= 0.0 else lambda_ * lr * lr / eta_max

    @torch.no_grad()
    def step(self, closure=None):
        loss = None
        if closure is not None:
            with torch.enable_grad():
                loss = closure()
        world = dist.get_world_size() if dist.is_initialized() else 1
        rank = dist.get_rank() if dist.is_initialized() else 0

        for group in self.param_groups:
            params = group["params"]
            params_pad = params + [torch.empty_like(params[-1])] * (-len(params) % world)
            for base in range(0, len(params), world):
                if base + rank < len(params):
                    p = params[base + rank]
                    if p.grad is None:
                        p.grad = torch.zeros_like(p)
                    state = self.state[p]
                    if len(state) == 0:
                        state["momentum_buffer"] = torch.zeros_like(p)
                    update = muon_direction(
                        p.grad, state["momentum_buffer"], beta=group["momentum"],
                        ns_steps=group["ns_steps"], nesterov=group["nesterov"],
                        shape_scale=group["shape_scale"])
                    # MuonSW: decay by 1 - lambda * lr^2 / eta_max.
                    p.mul_(max(0.0, 1.0 - self.scaled_decay_coeff(group)))
                    p.add_(update, alpha=-group["lr"])
                if dist.is_initialized():
                    dist.all_gather(params_pad[base:base + world], params_pad[base + rank])

        return loss
\end{lstlisting}
\end{center}
\restoregeometry

% \bibliographystyle{plain}
% \bibliography{references}

\end{document}